\documentclass{article}

\usepackage{arxiv}
\usepackage{arxiv}
\usepackage[utf8]{inputenc} 
\usepackage[T1]{fontenc}    
\usepackage[colorlinks, linkcolor=red, anchorcolor=green, citecolor=blue]{hyperref}
\usepackage{url}            
\usepackage{booktabs}       
\usepackage{amsfonts}       
\usepackage{nicefrac}       
\usepackage{microtype}      
\usepackage{graphicx}
\graphicspath{ {./images/} }
\usepackage{natbib}
\usepackage{mathptmx} 
\usepackage{times} 
\usepackage{amsmath} 
\usepackage{amssymb}  

\usepackage{amsthm}

\usepackage{color}
\usepackage{multirow}
\usepackage{algorithm}
\usepackage{algorithmic}
\usepackage{bm}
\usepackage{bbm}
\usepackage{float} 
\usepackage{subfigure}
\usepackage[cal=cm]{mathalfa} 
\newtheorem{definition}{Definition}
\newtheorem{proposition}{Proposition}
\newtheorem{theorem}{Theorem}
\newtheorem{lemma}{Lemma}
\newtheorem{corollary}{Corollary}
\newtheorem{assumption}{Assumption}

\title{Temporal Difference Learning as Gradient Splitting}

\author{
 Rui Liu \\
  Division of Systems Engineering\\
  Boston University\\
  Boston, MA, 02215 \\
  \texttt{rliu@bu.edu} \\
  
   \And
 Alex Olshevsky \\
  Department of ECE and Division of Systems Engineering\\
  Boston University\\
  Boston, MA, 02215 \\
  \texttt{alexols@bu.edu} \\
}

\begin{document}
\maketitle
\begin{abstract}
Temporal difference learning with linear function approximation is a popular method to obtain a low-dimensional approximation of the value function of a  policy in a Markov Decision Process. We give a new interpretation of this method in terms of a splitting of the gradient of an appropriately chosen function. As a consequence of this interpretation, convergence proofs for gradient descent can be applied almost verbatim to temporal difference learning. Beyond giving a new, fuller explanation of why temporal difference works, our interpretation also yields improved convergence times. We consider the setting with $1/\sqrt{T}$ step-size, where   previous comparable finite-time convergence time bounds for temporal difference learning had the multiplicative factor $1/(1-\gamma)$ in front of the bound, with $\gamma$ being the discount factor.  We show that a minor variation on TD learning which estimates the mean of the value function separately has a convergence time where  $1/(1-\gamma)$ only multiplies an asymptotically negligible term.
\end{abstract}


\section{Introduction}
Reinforcement learning is a basic machine learning paradigm which concerns learning optimal policies in Markov Decision Processes (MDP). It has been applied to many challenging practical problems, such as, autonomous driving \citep{chen2015deepdriving}, robotics \citep{gu2017deep}, bidding and advertising\citep{jin2018real}, and games \citep{silver2016mastering}. An important problem in reinforcement learning  is to estimate the value function for a given  policy, often referred to as the policy evaluation problem. Temporal difference (TD) learning originally proposed by \citet{sutton1988learning} is one of the most widely used policy evaluation algorithms. 
TD uses differences in predictions over successive time steps to drive the learning process, with  the prediction at any given time step updated via a carefully chosen step-size to bring it closer to the prediction of the same quantity at the next time step.

Despite its simple implementation, theoretical analysis of TD can be involved. This is particularly true when TD methods are applied to problems with large  state-spaces by maintaining an approximation to the value function.  Precise conditions for the asymptotic convergence of TD with linear function approximation were established by viewing TD as a stochastic approximation for solving a suitable Bellman equation in \citep{tsitsiklis1997analysis}. Before the last few years, there have been few non-asymptotic analyses of TD methods. The first non-asymptotic bounds for TD(0) with linear function approximation were given by \citet{korda2015td}, obtaining an exponential convergence rate for the centered variant of TD(0) when the underlying Markov chain mixes fast. However, some issues with the proofs of \citet{korda2015td} were listed by the subsequent work of \citet{narayanan2017finite}.  

 In \citet{lakshminarayanan2018linear} it was shown that TD algorithms with a problem independent constant step size and iterate averaging, achieve a problem dependent error that decays as $O(1/t)$ with the number of iterations $t$.  Convergence rates in probability with an $O(1/t)$ step-size were provided by  \citet{dalal2017finite}. Both analyses of \citet{dalal2017finite} and \citet{lakshminarayanan2018linear} assume samples used by the algorithm are i.i.d. rather than a trajectory in Markov chain. 
For the Markov chain observation model, \citet{bhandari2018finite} provide a $O(1 / \sqrt{T})$ convergence rate with step-size that scales as $1/\sqrt{T}$ and $O((\log t)/t)$ convergence rate with step size $O(1/t)$ for projected TD algorithm. The constant factors in the latter bounds depend on  $1/(1-\gamma)$, where $\gamma$ is the discount factor; this scaling is one of the things we will be studying in this paper. 

A number of papers also work on algorithms related to and inspired by the classic TD algorithm in the setting with Markovian sampling. \citet{srikant2019finite} give finite-time bounds for the TD algorithms with linear function approximation and a constant step-size. 
The two time-scale TD with gradient correction algorithm under a Markovian sampling and linear function approximation are discussed by \citet{xu2019two} and shown to converge as fast as $O((\log t)/t^{2/3})$. A method called TD-AMSGrad under linear function approximation is studied by \citet{xiong2020non};  with a constant step size, TD-AMSGrad converges to a neighborhood of the global optimum at a rate of $O(1/t)$, and with a diminishing step size, it converges exactly to the global optimum at a rate of $O((\log t)/t)$.

In this paper, we will study the convergence of with linear function approximation under Markov  observations. Our main contribution is to provide a new interpretation of temporal difference learning: we show how to view it as a ``splitting'' (a term we introduce and define later) of an appropriately chosen quadratic form.  As a consequence of this interpretation, it is possible to apply convergence proofs for gradient descent almost verbatim to temporal difference learning. 

The convergence times bounds we  obtain this way improve on existing results. In particular, we study step-sizes of $1/\sqrt{T}$, which are typically recommended because the resulting error bounds do not depend on the inverse eigenvalues of the matrices involved in the linear approximation, which can be quite large; by contrast, methods that achieve faster than $O(1/\sqrt{T})$ decay have performance guarantees that scale with these same eigenvalues. We provide a minor variation on TD(0)  for which we obtain a convergence rate that scales as $\left[ (1/(1-\gamma))^2 \right]/T + O(1/\sqrt{T})$, with the constant in the $O(\cdot)$ notation not blowing up as $\gamma \rightarrow 1$. We will also explain why a factor of $1/(1-\gamma)^2$ multiplying the asympotically negligible $O(1/T)$ term as here is unavoidable. 


\section{PRELIMINARIES}

In this section, we describe the basics of MDPs and TD learning methods. While all this material is standard and available in textbooks (e.g., \cite{sutton2018reinforcement}), it is necessary to standardize notation and make our presentation self-contained. 

\subsection{Markov Decision Processes}
We consider a discounted reward MDP described by a 5-tuple $(\mathcal{S},\mathcal{A},\mathcal{P},r,\gamma)$, where $\mathcal{S}=[n]=\{1,2,\cdots,n\}$ is the finite state space, $\mathcal{A}$ is the finite action space, $\mathcal{P}=(\mathcal{P}(s'|s,a))_{s,s' \in \mathcal{S}, a \in \mathcal{A}}$ are the transition probabilities, $r=(r(s,a,s'))_{s,s' \in \mathcal{S}, a \in \mathcal{A}}$ are deterministic rewards and $\gamma \in (0,1)$ is the discount factor. The (stationary) policy to be evaluated is a mapping $\mu$: $\mathcal{S} \times \mathcal{A} \rightarrow [0,1]$, where $\mu(s,a)$ are the probabilities to select action $a$ when in state $s$ and $\sum_{a \in \mathcal{A}} \mu(s,a) =1$ for all states $s \in \mathcal{S}$.  We adopt the shorthand $s_t$ for the state at step $t$, $a_t$ for the action taken at step $t$, and  $r_{t+1}=r(s_t,a_t,s_{t+1})$. 

The value function of the policy $\mu$, denoted $V^{\mu}: \mathcal{S} \rightarrow \mathbb{R}$ is defined as  $$V^{\mu}(s)=E_{\mu,s}\left[\sum_{t=0}^{\infty} \gamma^tr_{t+1}\right],$$ where $E_{\mu,s} \left[ \cdot \right]$ indicates that $s$ is the initial state and the actions are chosen according to $\mu$. 

The immediate reward vector $R^{\mu} : \mathcal{S} \rightarrow \mathbb{R}$ is defined as $$R^{\mu}(s)=E_{\mu,s}(r_1)=\sum_{s' \in \mathcal{S}} \sum_{a \in \mathcal{A}} \mu(s,a)\mathcal{P}(s'|s,a)r(s,a,s').$$ For the remainder of the paper, we will be fixing the policy $\mu$; consequently, we can talk about the probability transition matrix $P^{\mu}$ defined as  $$P^{\mu}(s,s')=\sum_{a \in \mathcal{A}} \mu(s,a)\mathcal{P}(s'|s,a).$$

In the following, we will treat $V^{\mu}$ and $R^{\mu}$ as vectors in $\mathbb{R}^n$, and treat $P^{\mu}$ as a matrix in $\mathbb{R}^{n \times n}$. It is well-known that $V^{\mu}$ satisfies the so-called Bellman equation \citep{sutton2018reinforcement}: 
defining the Bellman operator  $T^{\mu} : \mathbb{R}^n \rightarrow \mathbb{R}^n$ as
\begin{equation}\label{eq:operator}
     (T^{\mu}V^{\mu})(s) = \sum_{s'=1}^n P^{\mu}(s,s')(r(s,s')+\gamma V^{\mu}(s'))
\end{equation}for $s \in [n]$, where $r(s,s')=\sum_{a \in \mathcal{A}} \mu(s,a)r(s,a,s')$, we can then write Bellman equation as $T^{\mu}V^{\mu} = V^{\mu}$ \citep{sutton2018reinforcement}. 

Next, we state some standard assumptions from the literature. The first assumption is  on the underlying Markov chain.
\begin{assumption} \label{ass:mc}
 The Markov chain whose transition matrix is the matrix $P^{\mu}$ is irreducible and aperiodic. 
\end{assumption}
Following this assumption, the Markov decision process induced by the policy $\mu$ is ergodic with a unique stationary distribution $\pi  = (\pi_1, \pi_2, \cdots, \pi_n)$, a row vector whose entries are non-negative and sum to $1$. It also holds that $\pi_{s'} = \lim_{t \rightarrow \infty} (P^{\mu})^t(s,s')$ for any two states $s,s' \in [n]$. {\em Note that we are using $\pi$ to denote the stationary distribution of $P^{\mu}$, and not the policy (which is denoted by $\mu$). }


We will use the notation $r_{\max}$ to denote an upper bound on the rewards; more formally, $r_{\rm max}$ is a real number such that  $$|r(s,a,s')| \leq r_{{\rm max}} \mbox{ for all } s,s' \in [n], a \in \mathcal{A}.$$ Since the number of actions and states is finite, such an $r_{\rm max}$ always exists. 

We next introduce some notation that will make our analysis  more concise. For a symmetric positive definite matrix $A \in \mathbb{R}^{n \times n}$, we define the inner product $\langle x, y \rangle_{A} = x^T A y$ and the associated norm $\|x\|_A = \sqrt{x^T A x}$. Let $D = {\rm diag}(\pi_1,\cdots,\pi_n)$ denote the diagonal matrix whose elements are given by the entries of the stationary distribution $\pi$.  Given value functions $V$ and $V'$ on the state space $\mathcal{S}$, we have $$\left \langle V,V' \right \rangle _{D} = V^T D V'=\sum_{s \in \mathcal{S}}\pi_s V(s) V'(s),$$and the associated norm$$\|V\|_{D}^2=V^T D V = \sum_{s \in \mathcal{S}} \pi_s V(s)^2.$$

Finally, we define the Dirichlet seminorm, which is often called the Dirichlet form in the Markov chain literature \citep{diaconis1996logarithmic}; we follow here the notation of \cite{ollivier2018approximate}. The Dirichlet seminorm depends both on the transition matrix $P^{\mu}$ and the invariant measure $\pi$:
\begin{equation}\label{eq:def_Dir}
\|V\|_{{\rm Dir}}^2 = \frac{1}{2} \sum_{s,s' \in \mathcal{S} } \pi_s P^{\mu}(s,s') (V(s')-V(s))^2.
\end{equation}
It  is easy to see that, as a consequence of Assumption \ref{ass:mc},  $\|V\|_{{\rm Dir}}=0$ if and only if $V$ is a multiple of the all-ones vector. 

Similarly, we introduce the $k$-step Dirichlet seminorm, defined as 
 $$\|V\|_{{\rm Dir},k}^2 = \frac{1}{2} \sum_{s,s' \in \mathcal{S}} \pi_s (P^{\mu})^k(s,s')(V(s')-V(s))^2.$$

\subsection{Policy Evaluation, Temporal Difference Learning, and Linear Function Approximation}
Policy evaluation refers to the problem of estimating the value function $V^{\mu}$ for a given stationary policy $\mu$. If the size of the state space is large, computing $V^{\mu}(s)$ for all states $s$ may be prohibitively expensive.   A standard remedy is to use low dimensional approximation $V^{\mu}_{\theta}$ of $V^{\mu}$ in the classical TD algorithm as in \citet{sutton1988learning, sutton2018reinforcement}. For brevity, we omit the superscript $\mu$ throughout from now on. 

The classical TD($0$) algorithm with function approximation $V_{\theta}$ starts with an arbitrary value of the parameters $\theta_0$; upon observing the $t^{{\rm th}}$ transition $s_t \rightarrow s'_t$, it computes the scalar-valued temporal-difference error, $$\delta_t=r(s_t,a_t,s'_t)+\gamma V_{\theta_t}(s'_{t})-V_{\theta_t}(s_t),$$ and updates the parameter vector as 
\begin{equation}\label{eq:td}
    \theta_{t+1} = \theta_t + \alpha_t \delta_t \triangledown V_{\theta_t}(s_t).
\end{equation} Here  $\triangledown V_{\theta_t}(s_t)$ denotes the gradient of the function  $V_{\theta}(s_t)$ w.r.t to $\theta$ evaluated at $\theta=\theta_t$, and $\alpha_t$ is the  step size. Intuitively, updating in the direction $\delta_t \triangledown V_{\theta_t}(s_t)$ moves $V_{\theta_t}(s_t)$ closer to the bootstrapped value of $r(s_t,a_t,s'_t)+\gamma V_{\theta_t}(s'_{t})$.

We will be considering the TD($0$) algorithm with a linear function approximation $V_{\theta}$  defined as
\begin{equation}\label{eq:V_theta}
    V_{\theta}(s)=\sum_{l=1}^K \theta_l \phi_l(s) \quad \forall s \in \mathcal{S},
\end{equation} for a given set of $K$ feature vectors $\phi_l: \mathcal{S} \rightarrow \mathbb{R}$, $l \in [K]$. For each state $s$, we will define the vector  $\phi(s)$ which stacks up the features of $s$ as $\phi(s) = (\phi_1(s), \phi_1(s), \cdots, \phi_K(s))^T \in \mathbb{R}^K$. 
Finally, $\Phi \in \mathbb{R}^{n \times K}$ is defined to be the matrix $\Phi = [\phi_1, \cdots, \phi_K]$.

We thus have that $V_{\theta}(s) = \theta^T \phi(s)$ and the approximate TD($0$) update becomes
\begin{equation}\label{eq:lineartd}
     \theta_{t+1}=\theta_t + \alpha_t \left(r_t + \gamma \theta_t^T \phi(s'_t) - \theta_t^T \phi(s_t)\right) \phi(s_t).
\end{equation}

Next we state a common assumption on the  feature vectors, which requires that features used for approximation are linearly independent  \citep{tsitsiklis1997analysis,bhandari2018finite}.
\begin{assumption}\label{ass:features}
The matrix $\Phi$ has full column rank, i.e., the feature vectors $\{ \phi_1, \ldots, \phi_K\}$ are linearly independent. Additionally, we also assume that $ \|\phi(s)\|_2^2 \leq 1$ for $s \in \mathcal{S}$.
\end{assumption}

It is always possible to make sure this assumption holds. If the norm bound is unsatisfied, then the standard approach is to normalize the feature vectors so that it does. If the matrix $\Phi$ does not have full column rank, one can simply omit enough feature vectors so that it does; since the range of $\Phi$ is unaffected, this does not harm the quality of approximation.


It is well-known that under Assumptions \ref{ass:mc}-\ref{ass:features} as well as an additional assumption on the decay of the step-sizes $\alpha_t$, temporal difference learning converges with probability one;  furthermore, its limit is the fixed point of a certain projected Bellman equation
\citep{tsitsiklis1997analysis}. Henceforth we will use $\theta^*$ to denote this fixed point.

It is convenient to introduce the notation 
\begin{equation}\label{eq:def_gt}
    g_t(\theta) = \left(r_t + \gamma \phi(s'_t)^T \theta - \phi(s_t)^T \theta \right) \phi(s_t). 
\end{equation} for the direction taken by TD(0) at time $t$. Note that $g_t(\theta)$ is a scalar multiple of $\phi(s_t)$, the feature vector of the state encountered at time $t$.

Furthermore, $\bar{g}(\theta)$ will denote the average of $g_t(\theta)$ when the state is sampled according to the stationary distribution: 

$$\bar{g}(\theta) = \sum_{s,s' \in \mathcal{S}} \pi(s)P(s,s')\left(r(s,s')+\gamma \phi(s')^T \theta -\phi(s)^T \theta\right)\phi(s).$$Naturally it can be seen  (see \citet{tsitsiklis1997analysis}) that
\begin{equation} \bar{g}(\theta^*) = 0. \label{eq:gbarthetastar}
\end{equation}

\subsection{Eligibility traces}
We will also study a larger class of algorithms, denoted by TD($\lambda$) and parameterized by $\lambda \in [0,1]$, that contains as a special case the TD(0) algorithm discussed above. While TD($0$)  makes parameter updates in the direction of the (scaled) last feature vector  $g_t(\theta_t)$,  
the TD($\lambda$) algorithm maintains the  ``eligibility trace''  $$z_t = \sum_{k=-\infty}^t (\gamma \lambda)^k \phi(s_{t-k}),$$ which is a geometric weighted average of the feature vectors at all previously visited states, and takes a step in the direction of $z_t$. 

In practice, the sum will start at $k=0$ (or some other finite time); however, parts of the analysis are done with the sum starting at negative infinity because many of the results are much simpler in this setting, and doing so introduces only an exponentially decaying error term.

It is shown in \citet{tsitsiklis1997analysis} that, subject to Assumptions \ref{ass:mc}-\ref{ass:features} and appropriate decay of step-sizes,  TD($\lambda$) converges with probability one, and its limit is a fixed point of a certain projected \& averaged Bellman equation. We will denote this limit by $\theta_{\lambda}^*$. 

\subsection{Markov Chain Observation Model}
In this paper, we are interested in TD in the setting where the data is collected from a
single sample path of a Markov chain. Our final assumption is that the Markov chain mixes at a uniform geometric rate..

\begin{assumption}\label{ass:mix}
There are constants $m > 0$ and $\rho \in (0, 1)$ such that
\begin{equation*}
    \sup_{s \in \mathcal{S}} d_{{\rm TV}} (P^t(s,\cdot),\pi) \leq m \rho^t \quad t \in \mathbb{N}_0,
\end{equation*}where $d_{{\rm TV}}(P,Q)$ denotes the total-variation distance between probability measures $P$ and $Q$. In addition, the initial distribution of $s_0$ is the steady-state distribution $\pi$, so $(s_0, s_1, \cdots)$ is a stationary sequence.
\end{assumption}

Under Assumption \ref{ass:mc}, i.e., for irreducible and aperiodic Markov chains, the uniform mixing assumption always holds \citep{levin2017markov}. 
The assumption that the chain begins in steady-state allows us to simplify many mathematical expressions. It is worth noting that the assumption that $s_0$ is the the stationary distribution is primarily done to make the analysis and results tidier: given the uniform mixing assumption, one can apply analysis after the Markov chain is close to its steady-state.

\section{Our main result: a new interpretation of temporal difference learning}


All existing analysis temporal difference learning proceed by comparing it, either explicitly or implicitly, to the expected update, usually referred to as the mean-path update; for TD(0), this is 
\begin{equation}\label{eq:mean_path_TD}
     \theta_{t+1}=\theta_t + \alpha_t \bar{g}(\theta_t).
\end{equation}

Stochastic approximation \cite{robbins1951stochastic} is a common tool to make this comparison. Generally, one wants to ague that the mean-path TD update brings $\theta_t$ closer to its final value $\theta^*$.

The first theoretical analysis of TD(0) in  \citet{tsitsiklis1997analysis} proceeded based on the observation  that $\bar{g}(\theta)$ forms a positive angle with $\theta^* - \theta$, that is  \begin{equation} \label{eq:posnagle} \bar{g}(\theta)^T (\theta^* - \theta) > 0.
\end{equation}  An explicit version of this inequality was used in \cite{bhandari2018finite} where it is stated as Lemma 3:  
\begin{equation}\label{eq:lem3}
    \bar{g}(\theta)^T (\theta^* - \theta) \geq (1-\gamma) \|V_{\theta^*} - V_{\theta}\|_D^2.
\end{equation}

Our main result is a new interpretation of the quantity $\bar{g}(\theta)$ which explains why such an inequality holds, as well as allows us to derive stronger results. To do this, we first introduce the concept of a ``gradient splitting.'' 

\begin{definition} Let $A$ be a symmetric positive semi-definite matrix. A linear function $h(\theta) = B(\theta-a)$ is called a gradient splitting of the quadratic  $f(\theta) = (\theta - a)^T A (\theta-a)$ if 
$$B+B^T=2A.$$ 
\end{definition}

To the best of our knowledge, the concept is introduced here for the first time. We next explain why it is useful. 

\subsection{Gradient splitting and gradient descent}

Observe first that, as one should expect from the name, $(1/2) \nabla f(\theta)$ is a splitting of the gradient of $f$ since 
\[ \frac{1}{2} \nabla f(\theta) = A (\theta - a). \] 
Of course, it is far from the only splitting, since there are many $B$ that satisfy $B+B^T = 2A$. In particular, $B$ may be non-symmetric. For example, one can take $B$ to be equal to the upper triangular part of $2A$ plus the diagonal of $A$.   

The key property of splittings that make them useful is the following. 

\begin{proposition} Suppose $h(\theta)$ is a splitting of  $f(\theta)$. Then  
\begin{align*} (\theta_1 - \theta_2)^T \left( h(\theta_1) - h(\theta_2) \right) & = \frac{1}{2} (\theta_1 - \theta_2)^T \left( \nabla f(\theta_1) - \nabla f (\theta_2) \right).  
\end{align*} 
 \label{prop:splitmotivation}
\end{proposition}
\begin{proof} Indeed, 
\begin{align*} (\theta_1 - \theta_2)^T \left( h(\theta_1) - h(\theta_2) \right) 
 = & (\theta_1 - \theta_2)^T B (\theta_1 - \theta_2) \\ 
 = & \frac{1}{2} (\theta_1 - \theta_2)^T B (\theta_1 - \theta_2) + \frac{1}{2} (\theta_1 - \theta_2)^T B^T (\theta_1 - \theta_2) \\ 
 = & (\theta_1 - \theta_2)^T A (\theta_1 - \theta_2) \\ 
 = & \frac{1}{2} (\theta_1 - \theta_2)^T (\nabla f(\theta_1) - \nabla f(\theta_2)).
\end{align*} 
\end{proof} 

Thus, while $h(\theta)$ may be quite different from $\nabla f(\theta)$, the difference disappears once one looks at the inner products considered in Proposition \ref{prop:splitmotivation}. 

A particular consequence of Proposition \ref{prop:splitmotivation} can be obtained by plugging in $\theta_1 = a$, the global minimizer of $f(\theta)$. In that case, $\nabla f(a) = 0$ and $h(a)=0$ as well, and we obtain that for all $\theta$, 
\begin{equation} \label{eq:posprod} (a - \theta)^T h(\theta) = (a - \theta)^T \nabla f(\theta). \end{equation}

Thus {\em the splitting $h(\theta)$ has the exact same same angle with the ``direction to the optimal solution'' $a - \theta$ as the true gradient. } This is the punchline of this discussion.

Most analysis of gradient descent on convex functions are ultimately based on the observation that gradient descent ``makes progress''  towards the optimal solution because it has a positive inner product with the direction to optimality. As a consequence of this discussion, the same argument can be applied to gradient splittings.

\subsection{Our contribution} 

We now come back to temporal difference learning. To analyze TD learning, it is tempting to see if we can write the TD(0) and TD($\lambda$) updates as gradient descent on some appropriately chosen function. Unfortunately, it is well-known (and easy to see) that this cannot work. Indeed, in the TD(0) case, it is possible to express the average direction $\overline{g}(\theta)$ as  $\overline{g}(\theta) = B (\theta - \theta^*)$ and in some cases the matrix $B$ is not symmetric; this linear map cannot be the gradient of anything since the non-symmetry of $B$ would contradict equality of partial derivatives (see \cite{maei2011gradient}).

Our main results show that the temporal difference direction can, however, be viewed as a splitting of the gradient of an appropriately chosen function.  

\begin{theorem}\label{thm:split} Suppose Assumptions \ref{ass:mc}-\ref{ass:features} hold. Then in the TD(0) update, 
   $-\bar{g}(\theta)$ is a splitting of the gradient of the quadratic  $$f(\theta) =  (1-\gamma)\|V_{\theta}-V_{\theta^*}\|_{D}^2 + \gamma \|V_{\theta}-V_{\theta^*}\|_{{\rm Dir}}^2.$$
\end{theorem}

\begin{theorem}\label{thm:lambda_iden} Suppose Assumptions \ref{ass:mc}-\ref{ass:features} hold. Then, in the TD($\lambda$) update, the negative of the expected update $-E[z_t]$  is a splitting of the gradient of the quadratic 
\begin{small}
\begin{equation*} 
f^{(\lambda)}(\theta)  =   (1-\gamma \kappa) \| V_{\theta} - V_{\theta_{\lambda}^*} \|_{D}^2 + (1-\lambda) \gamma ||V_{\theta} - V_{\theta_{\lambda}^*}||_{\rm Dir}^2 + (1-\lambda) \lambda \gamma^2 ||V_{\theta} - V_{\theta_{\lambda}^*}||_{\rm Dir, 2}^2  + (1-\lambda) \lambda^2 \gamma^3 ||V_{\theta} - V_{\theta_{\lambda}^*}||_{\rm Dir,3}^2 + \cdots,
\end{equation*}
\end{small}where 
  $\kappa = (1-\lambda)/(1 - \gamma \lambda)$.
\end{theorem}

The proof of these theorems can be found in the supplementary information. Assumptions \ref{ass:mc} and \ref{ass:features} are not particularly crucial: they are used only to be able to define the stationary distribution $\pi$ and the unique fixed point $\theta^*$. 

These results provide some new insights into why temporal difference learning works. Indeed, there is no immediate reason why the bootstrapped update of (\ref{eq:lineartd}) should produce a reasonable answer, and it is well known that the version with nonlinear approximation in (\ref{eq:td}) can diverge (see  \citet{tsitsiklis1997analysis}). Convergence analyses of TD learning rely on (\ref{eq:posnagle}), but the proof of this equation from \cite{tsitsiklis1997analysis} does not yield a conceptual reason for why it should hold. 

The previous two theorems provide such a conceptual reason. It turns out that TD(0) and TD($\lambda$) are, on  average, attempting to minimize the functions $f(\theta)$ and $f^{(\lambda)}(\theta)$ respectively by moving in direction of a gradient splitting.  Moreover, the functions $f(\theta)$ and $f^{(\lambda)}(\theta)$ are plainly convex (they are positive linear  combinations of convex quadratics), so that Equation (\ref{eq:posprod}) immediately explains why (\ref{eq:posnagle}) holds. As far as we are aware, these results are new even in the case when no function approximation is used, which can be recovered by simply choosing the vectors $\phi_i$ to be the unit basis vectors. 

These theorems are inspired by the recent preprint \citep{ollivier2018approximate}. It is shown there that,  if $P$ is reversible, then $-\bar{g}(\theta)$ is exactly the gradient of the function $f(\theta)$. Theorem \ref{thm:split} may be viewed as a way to generalize this observation to the non-reversible case. 

\section{Consequences}

We now discuss several consequences of our new interpretation of temporal difference learning. These will all be along the lines of improved convergence guarantees. Indeed, as we mentioned in the previous section, viewing TD learning as gradient splitting allows us to take existing results for gradient descent and ``port'' them almost verbatim to the temporal difference setting.

In the main body of the paper, we focus on TD(0); the case of TD($\lambda$) is discussed in the supplementary information. As mentioned earlier, existing analyses of TD(0) rely on (\ref{eq:posnagle}) as well as its refinement (\ref{eq:lem3}). However, as a consequence of Proposition \ref{prop:splitmotivation}, we can actually write out explicitly the inner product between the mean TD(0) direction $\bar{g}(\theta)$ and the direction to optimality $\theta^*-\theta$. 

\begin{corollary}\label{cor:identity} 
    For any $\theta \in \mathbb{R}^K$,  \begin{equation}\label{eq:identity}
        (\theta^* - \theta)^T \bar{g}(\theta) = (1-\gamma)\|V_{\theta^*}-V_{\theta}\|_{D}^2 + \gamma \|V_{\theta^*}-V_{\theta}\|_{{\rm Dir}}^2.
    \end{equation} 
\end{corollary}

\begin{proof} Indeed, we can use (\ref{eq:gbarthetastar})  to argue that 
\begin{align}
    (\theta^*-\theta)^T \bar{g}(\theta) & =   
     (\theta^*-\theta)^T \left(  - \bar{g} (\theta^*) - (- \bar{g}(\theta) ) ) \right) \nonumber \\
    & =  \frac{1}{2} (\theta^* - \theta)^T (\nabla f(\theta^*) - \nabla f(\theta)), \label{eq:fdiff}
\end{align} where the last step follows by Theorem \ref{thm:split} and Proposition \ref{prop:splitmotivation}; here $f(\theta)$ is the function from Theorem \ref{thm:split}.

However, for any quadratic function $q(\theta) = ( \theta - a)^T P (\theta - a)$ where $P$ is a symmetric matrix, we have that 
\[ (a-\theta)^T (\nabla q(a) - \nabla q(\theta)) = 2 q(\theta). \]
Applying this to the function $f(\theta)$ in (\ref{eq:fdiff}), we complete the proof. 
\end{proof} 

This corollary should be contrasted to (\ref{eq:posnagle}) and (\ref{eq:lem3}). It is clearly a strengthening of those equations. More importantly, this is an equality, whereas (\ref{eq:posnagle}) and (\ref{eq:lem3}) are inequalities. We thus see that the average TD(0) direction makes more progress in the direction of the optimal solution compared to the previously available bounds. 

In the remainder of the paper,  we will use Corollary \ref{cor:identity} to obtain improved convergence times for TD(0); also, a natural generalization of that Corollary which appeals to Theorem \ref{thm:lambda_iden} instead of Theorem \ref{thm:split}  results in improved convergence times for TD($\lambda$), as explained in the supplementary information. We focus on a particular property which is natural in this context: scaling with the discount factor $\gamma$. 

Indeed, as we discussed in the introduction, an undesirable feature of some of the existing analyses of temporal difference learning is that they scale multiplicatively with $1/(1-\gamma)$. It is easy to see why this should be so: existing analyses rely on variations of (\ref{eq:lem3}), and as $\gamma \rightarrow 0$, that equation guarantees smaller and smaller progress towards the limit. Unfortunately, it is natural to set the discount factor close to $1$ in order to avoid focusing on short-term behavior of the policy. 

But now we can instead rely on Corollary \ref{cor:identity} and this corollary suggests that as $\gamma \rightarrow 1$, the inner product between the expected TD(0) direction $\bar{g}(\theta)$ and the direction to the optimal solution $\theta^*-\theta$ will be lower bounded by $\gamma ||V_{\theta} - V_{\theta^*}||_{\rm Dir}^2$. A difficulty, however, is that the Dirichlet semi-norm can be zero even when applied to a nonzero vector. We next discuss the results we are able to derive with this approach.

\subsection{Improved error bounds}

As mentioned earlier, a nice consequence of the gradient splitting interpretation is that we can apply the existing proof for gradient descent almost verbatim to gradient splittings. In particular, temporal difference learning when the states are sampled i.i.d. could be analyzed by simply following existing analyses of noisy gradient descent. However, under our Markov observation model, it is not true that the samples are i.i.d; rather, we proceed by modifying the analysis of so-called {\em Markov Chain Gradient Descent}, analyzed in the papers \cite{sun2018markov, johansson2010randomized}. 

 One issue is the choice of step-size. The existing literature on temporal difference learning contains a range of possible step-sizes from $O(1/t)$ to $O(1/\sqrt{t})$ (see \cite{bhandari2018finite, dalal2017finite, lakshminarayanan2018linear}). A step-size that scales as $O(1/\sqrt{t})$ is often preferred because, for faster decaying step-sizes, performance will scale with the smallest eigenvalue of   $\Phi^T D \Phi$ or related quantity, and these can be quite small. This is not the case, however, for a step-size that decays like $O(1/\sqrt{t})$.

We will be using the standard notation 
$$\bar{\theta}_T = \frac{1}{T} \sum_{t=0}^{T-1} \theta_t,$$ to denote the running average of the iterates $\theta_t$. 

We will be considering the projected TD(0) update 
\begin{equation}\label{eq:projecttd}
    \theta_{t+1} = {\rm Proj}_{\Theta}(\theta_t + \alpha_t g_t(\theta_t)),
\end{equation} where $\Theta$ is a convex set containing the optimal solution $\theta^*$. Moreover, we will assume that the norm of every element in $\Theta$ is at most $R$. Setting $G=r_{\rm max}+2R$, we have the following error bound.

\begin{corollary}\label{thm:bound}
 Suppose Assumptions \ref{ass:mc}-\ref{ass:mix} hold. Suppose further that $(\theta_t)_{t \geq 0}$ is generated by the  Projected TD algorithm of (\ref{eq:projecttd}) with  $\theta^* \in \Theta$ and $\alpha_0 = \dots = \alpha_T = 1/\sqrt{T}$.  Then \begin{equation}
     E  \left[ (1-\gamma)\|V_{\theta^*}-V_{\bar{\theta}_T}\|_{D}^2 + \gamma \|V_{\theta^*}-V_{\bar{\theta}_T}\|_{{\rm Dir}}^2 \right]   \leq \frac{ \| \theta^*-\theta_{0}\|_2^2 + G^2\left[9 + 12\tau ^{{\rm mix}}\left(1/\sqrt{T}\right) \right]}{2\sqrt{T}}.\label{eq:bound}
 \end{equation}

\end{corollary}

The definition of $\tau ^{{\rm mix}}$ and proof is available in the supplementary information. We next compare this bound to the existing literature. The closest comparison is  Theorem 3(a) in \citep{bhandari2018finite} which shows that
 \begin{equation}
     E \left[\|V_{\theta^*}-V_{\bar{\theta}_T}\|_{D}^2 \right]  \leq \frac{ \| \theta^*-\theta_{0}\|_2^2}{2(1-\gamma)\sqrt{T}}   + \frac{G^2\left[9 + 12\tau ^{{\rm mix}}\left(1/\sqrt{T}\right) \right]}{2(1-\gamma)\sqrt{T}}.\label{eq:thm3} 
 \end{equation}

Corollary \ref{thm:bound} is stronger than this, because this bound can be derived from Corollary \ref{thm:bound} by ignoring the second term on the left hand side of (\ref{eq:bound}). Moreover, we next argue that Corollary \ref{thm:bound} is stronger an interesting way, in that it offers a new insight on the behavior of temporal difference learning. 

Observe that the bound of (\ref{eq:thm3}) blows up as $\gamma \rightarrow 1$. On the other hand, we can simply ignore the first term on the left-hand side of Corollary \ref{thm:bound} to obtain 
\begin{equation}
    E \left[\|V_{\theta^*}-V_{\bar{\theta}_T}\|_{{\rm Dir}}^2 \right]  \leq \frac{ \| \theta^*-\theta_{0}\|_2^2}{2 \gamma \sqrt{T}} +\frac{G^2\left[9 + 12\tau ^{{\rm mix}}\left(1/\sqrt{T}\right) \right]}{2 \gamma \sqrt{T}}. \label{eq:dir}
\end{equation} 

In particular, we see that $E \left[ ||V_{\theta^*} - V_{\bar{\theta}_T}||_{\rm Dir}^2 \right]$ does not blow up as $\gamma \rightarrow 1$. To understand this, recall that the Dirichlet semi-norm is equal to zero if and only if applied to a multiple of the all-ones vector. Consequently, $||V||_{\rm Dir}$ is properly thought of as norm of the projection of $V$ onto ${\bf 1}^\perp$. {\em We therefore obtain the punchline of this section:  the error of (averaged \& projected) temporal difference learning projected on ${\bf 1}^\perp$ does not blow up as the discount factor approaches 1.}

There are scenarios where this is already interesting. For example, if TD(0) is a subroutine of policy evaluation, it will be used for a policy improvement step, which is clearly unaffected by adding a multiple of the all-ones vector to the value function. Similarly, Proposition 4 of \cite{ollivier2018approximate} shows that the bias in the policy gradient computed from an approximation $\hat{V}$ to the true value function $V$ can be bounded solely in terms of $||V-\hat{V}||_{\rm Dir}^2$ (multiplied by a factor that depends  on how the policies are parameterized).

It is natural to wonder whether the dependence on $1/(1-\gamma)$ can be removed completely from bounds on the performance of temporal difference learning (not just in terms of projection on ${\bf 1}^\perp$). We address this next.



\subsection{Mean-adjusted temporal difference learning}\label{sec:3step}

Unfortunately, it is easy to see that the dependence on  $1/(1-\gamma)$ in error bounds for temporal difference learning cannot be entirely removed. We next give an informal sketch explaining why this is so. We consider the case where samples $s,s'$ are i.i.d. with probability $\pi_s P(s,s')$ rather than coming from the Markov chain  model, since this only makes the estimation problem easier.

\noindent {\bf Estimating the mean of the value function.}
Let us denote by $V$  the true value function; because it is the fixed point of the Bellman operator, $V =  R + \gamma P V$, we have that $$V = (I-\gamma P)^{-1} R = \left(\sum_{m=0}^{\infty}\gamma^m P^m \right) R. $$ Define $\bar{V} = \pi^T V$; then  \begin{equation} \label{eq:vbar} \bar{V} = \pi^T V = \pi^T\left( \sum_{m=0}^{\infty}\gamma^m P^m \right) R  = \frac{\pi^T R}{1- \gamma}.\end{equation}

Under i.i.d. sampling, what we have are samples from a random variable  $\tilde{R}$ which takes the value  $r(s,s')$ with probability $\pi_s P(s,s')$. From (\ref{eq:vbar}), we have that $(1-\gamma) \bar{V} = E [\tilde{R}]$. 

From $T$ samples of the scalar random variable $\tilde{R}$, the best estimate $\hat{R}_T$ will satisfy $E[(\hat{R}_T - E[\tilde{R}])^2] = \Omega(1/T)$ in the worst-case. This implies that the best estimator $\hat{V}$ of $\bar{V}$ will satisfy $E [ (\hat{V} - \bar{V})^2] = \Omega \left( (1/(1-\gamma)^2)/T \right)$. 

To summarize, the squared error in estimating just the mean of the value function will already scale with $1/(1-\gamma)$. If we consider e.g., $\Phi$ to be the identity matrix, in which case $V_{\theta^*}$ is just equal to the true value function, it easily follows that it is not possible to estimate $V_{\theta^*}$ with error that does not scale with $1/(1-\gamma)$. 

\noindent {\bf A better scaling with the discount factor.} Note, however, that the previous discussion implied that a term like $(1/(1-\gamma)^2)/T$ in a bound on the squared error is unavoidable. But with $1/\sqrt{T}$ step-size, the error will in general decay more slowly as $1/\sqrt{T}$ as in Corollary \ref{thm:bound}. Is it possible to derive a bound where the only scaling with $1/(1-\gamma)$ is in the asymptotically negligible $O(1/T)$ term?

As we show next, this is indeed possible. The algorithm is very natural given the  discussion in the previous subsection: we run projected and averaged TD(0) and estimate the mean of the value function separately, adjusting the outcome of TD(0) to have the right mean in the end.  Building on Corollary \ref{thm:bound}, the idea is that the mean will have expected square error that scales with $(1/(1-\gamma)^2)/T$ while the temporal difference method will estimate the projection onto ${\bf 1}^\perp$ without blowing up as $\gamma \rightarrow 1$. 

The pseudocode of the algorithm is given next as Algorithm 1 and Corollary \ref{thm:3step} bounds its performance. Note that, in contrast to the bounds in the last subsection, Theorem \ref{thm:3step} bounds the error to the true value function $V$ directly. This is a more natural bound for this algorithm which tries to directly match the mean of the true value function. 

\begin{algorithm}[ht] 
\caption{Mean-adjusted TD($0$)}
\begin{algorithmic}[1]
\STATE Initialize $\bar{A}_{-1}=0$, $s_0 = \pi$, and some initial condition $\theta_0$. \\
\FOR {$t=0$ to $T-1$} 
    \STATE Projected TD($0$) update:  \\ $\theta_{t+1}={\rm Proj}_{\Theta}\left(\theta_t + \alpha_t g_t(\theta_t)\right)$
    \STATE Keep track of the average reward: $\bar{A}_t = \frac{t \bar{A}_{t-1}+ r_t}{t+1}$
\ENDFOR
\STATE Set $\hat{V}_T = \frac{\bar{A}_T}{1-\gamma}$
\STATE Output $V'_T = V_{\bar{\theta}_T} + \left( \hat{V}_T - \pi^T V_{\bar{\theta}_T} \right) \bm{1}$
\end{algorithmic}
\end{algorithm}

\begin{corollary} \label{thm:3step}
 Suppose that $(\theta_t)_{t \geq 0}$ and $V_T'$ are generated by Algorithm 1 with step-sizes $\alpha_0 = \dots = \alpha_T = 1/\sqrt{T}$.  Suppose further that $\Theta$ is a convex set that contains $\theta^*$. Let $t_0$ be the largest integer which satisfies $t_0 \leq 2 \tau^{{\rm mix}}\left(\frac{1}{2(t_0+1)} \right)$. Then as long as $T \geq t_0$, we will have \begin{small}
 \begin{equation*} E \left[  \|V'_T-V\|_{D}^2 \right]
\leq   O \left(    \|V_{{\theta}^*}-V\|_{D}^2 + \frac{ r_{\rm max}^2 \tau^{{\rm mix}}\left(\frac{1}{2(T+1)}\right)}{(1-\gamma)^2 T}  
 +   \frac{  \| \theta^*-\theta_{0}\|_2^2 + G^2\left[1+\tau ^{\rm mix}(1/\sqrt{T})\right] }{\sqrt{T}}  \cdot \min \left \{ \frac{r(P)}{\gamma} , \frac{1}{1-\gamma}  \right \} \right),
\end{equation*} \end{small}where $r(P)$ is the inverse spectral gap of the additive reversibilization of the transition matrix $P$ (formally defined in the supplementary information).
\end{corollary}

The bound of Corollary \ref{thm:3step} has the structure  mentioned earlier: the blowup with $1/(1-\gamma)$ occurs only in the $\widetilde{O}(1/T)$. The last term, which scales like $\widetilde{O}(1/\sqrt{T})$ gets multiplied by the minimum of $1/(1-\gamma)$ and a quantity that depends on the matrix $P$, so that it does not blow up as $\gamma \rightarrow 1$.

Note that, unlike the previous bounds discussed in this paper, this bound depends on an eigenvalue gap associated with the matrix $P$. However, this dependence is in such a way that it only helps: when $1/(1-\gamma)$ is small, there is no dependence on the eigenvalue gap, and it is only when $\gamma \rightarrow 1$ that performance ``saturates'' at something depending on an eigenvalue gap.

\section{Conclusion}

We have given a new interpretation of temporal difference learning in terms of a splitting of gradient descent. As a consequence of this interpretation, analyses of gradient descent can apply to temporal difference learning almost verbatim. 

We have exploited this interpretation to observe that temporal difference methods learn the projection of the value function onto ${\bf 1}^\perp$ without any blowup as $\gamma \rightarrow 1$; by contrast, previous work tended to have error bounds that scaled with $1/(1-\gamma)$. While, as we explain, it is not possible to remove the dependence on $O(1/(1-\gamma))$ in general, we provide an error bound for a simple modification to TD(0) where the only dependence on $1/(1-\gamma)$ is in the asymptotically negligible term. 

An open problem might be to improve the scaling of the bounds we have obtained in this paper with $P$. Our focus has been on scaling with $1/(1-\gamma)$ but one could further ask what dependence on the transition matrix $P$ is optimal. It is natural to wonder, for example, whether the $r(P)$ factor in the last term of Corollary, \ref{thm:3step} measuring how the performance ``saturates'' as $\gamma \rightarrow 1$,  could be improved. Typically, error bounds in this setting scale with the (non-squared) spectral gap of $P$, which can be much smaller than $r(P)$. 




\bibliography{reference}
\bibliographystyle{apalike}
\newpage
\appendix

\section{Proofs of Theorems \ref{thm:split} and \ref{thm:lambda_iden}}

In this section, we present the detailed proof of the main results of our paper, i.e., Theorem \ref{thm:split} and \ref{thm:lambda_iden}. We begin with the proof of Theorem \ref{thm:split}.

Before the proof, we introduce some necessary notation. Let $\phi$ be the feature vector of a random state generated according to the stationary distribution $\pi$. In other words, $\phi = \phi(s_k)$ with probability $\pi_{s_k}$. Let $\phi'$ be the feature vector of the next state $s'$ and let $r=r(s,s')$. Thus $\phi \text{ and }\phi'$ are random vectors and $r$ is a random variable. As shown in Equation (2) of \citet{bhandari2018finite}, $\bar{g}(\theta)$ can be written as $$\bar{g}(\theta) = E\left[ \phi r \right] +  E\left[\phi (\gamma \phi'-\phi)^T\right]\theta.$$ With these notations in place, we begin the proof of  Theorem \ref{thm:split}.

\begin{proof}[Proof of Theorem \ref{thm:split}]
Recall that $\theta^*$ is the unique vector with $\bar{g}(\theta^*)=0$ (see  Lemma 6 in \citep{tsitsiklis1997analysis}). Consider
\begin{equation} \label{eq:barg}
    \bar{g}(\theta) =  \bar{g}(\theta) -  \bar{g}(\theta^*)= E\left[\phi (\gamma \phi'-\phi)^T \right](\theta-\theta^*).
\end{equation}
To conclude that $\bar{g}(\theta)$ is a splitting of the gradient for a quadratic form $f(\theta)$, we need to calculate the gradient of $f(\theta)$. Let us begin with the Dirichlet norm and perform the following sequence of manipulations:

\begin{align}
\|V_{\theta}-V_{\theta^*}\|_{{\rm Dir}}^2 =& \frac{1}{2} \sum_{s,s' \in \mathcal{S}} \pi(s)P(s,s')\left[V_{\theta^*}(s)-V_{\theta}(s)-V_{\theta^*}(s')+V_{\theta}(s')\right]^2 \nonumber \\
=&  \frac{1}{2} \sum_{s,s' \in \mathcal{S}} \pi(s)P(s,s')\left [\left(V_{\theta^*}(s)-V_{\theta}(s)\right)^2 + \left(V_{\theta^*}(s')-V_{\theta}(s')\right)^2\right] \nonumber \\ 
& -\sum_{s,s' \in \mathcal{S}} \pi(s)P(s,s') \left(V_{\theta^*}(s)-V_{\theta}(s)\right)\left(V_{\theta^*}(s')-V_{\theta}(s')\right) \nonumber \\
= & \frac{1}{2} \sum_{s \in \mathcal{S}} \pi(s)\left(\sum_{s' \in \mathcal{S}}P(s,s') \right)(V_{\theta^*}(s)-V_{\theta}(s))^2 \nonumber\\
&+  \frac{1}{2} \sum_{s' \in \mathcal{S}} \left(\sum_{s \in \mathcal{S}} \pi(s)P(s,s') \right)(V_{\theta^*}(s')-V_{\theta}(s'))^2 \nonumber\\
& -\sum_{s,s' \in \mathcal{S}} \pi(s)P(s,s') (\theta - \theta^*)^T \phi(s) \phi(s')^T (\theta - \theta^*) \nonumber \\
=&  \frac{1}{2} \sum_{s \in \mathcal{S}} \pi(s) \left(V_{\theta}(s)-V_{\theta^*}(s) \right)^2 + \frac{1}{2} \sum_{s' \in \mathcal{S}} \phi(s')\left(V_{\theta}(s')-V_{\theta^*}(s')\right)^2 \nonumber \\
& -(\theta - \theta^*)^T E\left[ \phi \phi'^T \right] (\theta - \theta^*) \nonumber \\
= & \|V_{\theta}-V_{\theta^*}\|_{D}^2-(\theta - \theta^*)^T E\left[\phi \phi'^T\right](\theta - \theta^*). \label{eq:diridentity}
\end{align} 

In the above sequence of equations, the first equality is just the definition of Dirichlet semi-norm; the second equality follows by expanding the square; the third equality follows by interchanging sums and the definition of $V_{\theta}$; the fourth equality uses that $\pi$ is a stationary distribution of $P$, as well as the definition of $\phi$ and $\phi'$; and the final equality uses the definition of the $||\cdot||_D$ norm.

Our next step is to use the identity we have just derived to rearrange the definition of  $\|V_{\theta}-V_{\theta^*}\|_{D}^2$:

\vfill
\newpage

\begin{align}
\|V_{\theta}-V_{\theta^*}\|_{D}^2 & =  (V_{\theta}-V_{\theta^*})^T D (V_{\theta}-V_{\theta^*}) = (\theta - \theta^*)^T \Phi^T D \Phi (\theta - \theta^*) \nonumber \\
& = (\theta - \theta^*)^T \sum_{s \in \mathcal{S}} \pi(s) \phi(s) \phi(s)^T (\theta - \theta^*) \nonumber \\
& = (\theta - \theta^*)^T E[ \phi \phi^T] (\theta - \theta^*).
\label{eq:dnorm} \end{align}

We now use these identities to write down  a new expression for the function  $f(\theta)$:
\begin{eqnarray*}
    f(\theta) & = & (1-\gamma) ||V_{\theta} - V_{\theta^*}||_{\rm D}^2 + \gamma ||V_{\theta} - V_{\theta^*}||_{\rm Dir}^2 \\ 
    & = & (1-\gamma) ||V_{\theta} - V_{\theta^*}||_{\rm D}^2 + \gamma \left(||V_{\theta} - V_{\theta^*}||_{\rm D}^2 - (\theta - \theta^*)^T E\left[\phi \phi'^T\right](\theta - \theta^*) \right) \\ 
    & = & ||V_{\theta} - V_{\theta^*}||_{\rm D}^2 - \gamma  (\theta - \theta^*)^T E\left[\phi \phi'^T\right](\theta - \theta^*)  \\ 
    & = & (\theta - \theta^*)^T E[ \phi \phi^T] (\theta - \theta^*) - \gamma  (\theta - \theta^*)^T E\left[\phi \phi'^T\right](\theta - \theta^*)  \\ 
    & = &  (\theta - \theta^*)^T E\left[\phi (\phi-\gamma \phi')^T \right] (\theta - \theta^*).  
\end{eqnarray*} 
In the above sequence of equations, the first equality is just the definition of the two norms; the second equality is obtained by plugging in (\ref{eq:diridentity}); the third equality is obtained by cancellation of terms; the fourth equality is obtained by plugging in (\ref{eq:dnorm}); and the last step follows by merging the two terms together.

As a consequence of writing $f(\theta)$ this way, we can write down a new expression for the gradient of $f(\theta)$ directly: 
\begin{equation}\label{eq:f_gradient}
    \triangledown f(\theta) = \left( E\left[\phi (\phi-\gamma \phi')^T  \right] + E \left[(\phi-\gamma \phi') \phi^T \right] \right) (\theta - \theta^*).
\end{equation}
Combining (\ref{eq:barg}) and (\ref{eq:f_gradient}), it is immediately that $-\bar{g}(\theta)$ is a splitting of $\triangledown f(\theta)$.
\end{proof}

We next turn to the proof of Theorem  \ref{thm:lambda_iden}. Before beginning the proof, we introduce some notations. 

The operator $T^{(\lambda)}$ is defined as:
\begin{equation}\label{eq:td_lam_op}
    \left(T^{(\lambda)} J\right)(s) = (1-\lambda) \sum_{m=0}^{\infty} \lambda^m E \left[ \sum_{t=0}^{m} \gamma^t r(s_t,s_{t+1})+\gamma^{m+1} J(s_{m+1})| s_0 = s \right]
\end{equation}for vectors $J \in \mathbb{R}^n$. The expectation is taken over sample paths taken by following actions according to policy $\mu$; recalling that this results in the transition matrix $P$, we can write this as \begin{equation}
    T^{(\lambda)} J = (1-\lambda) \sum_{m=0}^{\infty} \lambda^m \sum_{t=0}^{m} \gamma^t P^t R + (1-\lambda) \sum_{m=0}^{\infty} \lambda^m \gamma^{m+1} P^{m+1} J. \label{eq:tlambda}
\end{equation}

We next devise new notation that is analogous to the TD(0) case. Let us denote the quantity $\delta_t z_t$ by $x(\theta_t,z_t)$ and its steady-state mean by $\bar{x}(\theta)$. It is known that
\begin{equation}\label{eq:xbar}
    \bar{x}(\theta) = \Phi ^T D \left(T^{(\lambda)} (\Phi \theta) -\Phi \theta\right),
\end{equation} see Lemma 8 of \citet{tsitsiklis1997analysis}
; it also shown there that TD($\lambda$) converges to a unique fixed point of a certain Bellman equation which we'll denote by $\theta_{\lambda}^*$, and which satisfies 
\begin{equation}\label{eq:xbar^*}
    \bar{x}(\theta_{\lambda}^*)=0.
\end{equation}

With these preliminaries in place, we can begin the proof. 

\begin{proof}[Proof of Theorem \ref{thm:lambda_iden}]
By the properties of $T^{(\lambda)}$ and $\bar{x}(\theta)$ given in (\ref{eq:xbar}) and (\ref{eq:td_lam_op}) respectively, we have
\begin{align}
   \bar{x}(\theta) &= \bar{x}(\theta) - \bar{x}(\theta_{\lambda}^*) \nonumber  \\
   &= \Phi^T D \left( T^{(\lambda)} \left( \Phi \theta \right) - \Phi \theta \right) - \Phi^T D \left( T^{(\lambda)} \left( \Phi \theta_{\lambda}^* \right) - \Phi \theta_{\lambda}^* \right) \nonumber  \\
   & =  \Phi^T D \left(T^{(\lambda)} (\Phi \theta) - T^{(\lambda)} (\Phi \theta_{\lambda}^*) -\Phi (\theta-\theta_{\lambda}^*)\right) \nonumber  \\
   & = \left[ (1-\lambda) \sum_{m=0}^{\infty} \lambda^m \gamma^{m+1} \Phi ^T D P^{m+1} \Phi - \Phi ^T D \Phi \right] (\theta - \theta_{\lambda}^*), \label{eq:tdoptdiff}
\end{align} where the last line used (\ref{eq:tlambda}).

Our next step is to derive a convenient expression for $f^{(\lambda)}(\theta)$. We begin by finding a clean expression for the Dirichlet form that appears in the definition of  $f^{(\lambda)}(\theta)$:
\begin{align}
      \|V_{\theta} -V_{\theta_{\lambda}^*} \|_{ {\rm Dir}, m+1}^2 
    = &  \frac{1}{2} \sum_{s,s' \in \mathcal{S}} \pi_s P^{m+1}(s,s')(V_{\theta}(s) -V_{\theta_{\lambda}^*}(s)-V_{\theta}(s') +V_{\theta_{\lambda}^*}(s'))^2  \nonumber \\
   =& \frac{1}{2} \sum_{s,s' \in \mathcal{S}}  \pi_s P^{m+1}(s,s') \left[ (V_{\theta}(s) -V_{\theta_{\lambda}^*}(s))^2 + (V_{\theta}(s') -V_{\theta_{\lambda}^*}(s'))^2\right] \nonumber  \\
   & - \sum_{s,s' \in \mathcal{S}}  \pi_s P^{m+1}(s,s')\left(V_{\theta}(s) -V_{\theta_{\lambda}^*}(s)\right)\left(V_{\theta}(s') -V_{\theta_{\lambda}^*}(s')\right)  \nonumber  \\
   =& \frac{1}{2} \sum_{s \in \mathcal{S}} \pi_s \left(\sum_{s' \in \mathcal{S}} P^{m+1}(s,s')\right) (V_{\theta}(s) -V_{\theta_{\lambda}^*}(s))^2 + \frac{1}{2} \sum_{s' \in \mathcal{S}}   \left(\sum_{s \in \mathcal{S}} \pi_s P^{m+1}(s,s')\right)(V_{\theta}(s') -V_{\theta_{\lambda}^*}(s'))^2 \nonumber  \\
   & - \sum_{s,s' \in \mathcal{S}}  \pi_s P^{m+1}(s,s')\left(V_{\theta}(s) -V_{\theta_{\lambda}^*}(s)\right)\left(V_{\theta}(s') -V_{\theta_{\lambda}^*}(s')\right) \nonumber  \\
   =& \frac{1}{2} \sum_{s \in \mathcal{S}}  \pi_s (V_{\theta}(s) -V_{\theta_{\lambda}^*}(s))^2+ \frac{1}{2} \sum_{s' \in \mathcal{S}}   \pi_{s'} (V_{\theta}(s') -V_{\theta_{\lambda}^*}(s'))^2 \nonumber  \\
   & - \sum_{s,s' \in \mathcal{S}}  \pi_s P^{m+1}(s,s')\left(V_{\theta}(s) -V_{\theta_{\lambda}^*}(s)\right)\left(V_{\theta}(s') -V_{\theta_{\lambda}^*}(s')\right) \nonumber  \\
   =&   \sum_{s \in \mathcal{S}} \pi_s (V_{\theta}(s) -V_{\theta_{\lambda}^*}(s))^2
    -   \sum_{s \in \mathcal{S}} \pi_s \left(V_{\theta}(s) -V_{\theta_{\lambda}^*}(s)\right)
   \sum_{s' \in \mathcal{S}} P^{m+1}(s,s')
    \left(V_{\theta}(s') -V_{\theta_{\lambda}^*}(s')\right) \nonumber  \\
   =&  (\theta - \theta_{\lambda}^*)^T  \left(\Phi^T D \Phi -  \Phi^T D P^{m+1} \Phi \right) (\theta - \theta_{\lambda}^*). \label{eq:direxpr}
\end{align}

In the above sequence of equations, the first equality follows by the definition of the $m+1$-Dirichlet norm; the second equality follows by expanding the square; the third equality follows by interchanging the order of summations; the fourth equality uses that any power of a stochastic matrix is stochastic, and the $\pi P^{m+1} = \pi$; the fifth equality combines terms and rearranges the order of summation; and the last line uses the definition $V_{\theta} = \Phi \theta$. 

We'll also make use of the obvious identity \begin{equation} \label{eq:vthetadef} \| V_{\theta} - V_{\theta_{\lambda}^*} \|_{D}^2 = ({\theta} -{\theta_{\lambda}^*})^T \Phi^T D \Phi ({\theta} -{\theta_{\lambda}^*}).\end{equation} 

Putting all this together, we can express the function  $f^{(\lambda)}(\theta)$ as:
\begin{small}
\begin{align*}
    f^{(\lambda)}(\theta) &=   
    (1-\gamma \kappa) ||V_{\theta} - V_{\theta_{\lambda}^*}||_D^2 + (1-\lambda) \sum_{m=0}^{+\infty} \lambda^m \gamma^{m+1} ||V_{\theta} - V_{\theta_{\lambda}^*}||_{\rm Dir, m+1}^2 \\ 
    & = ({\theta} -{\theta_{\lambda}^*})^T \left[(1- \gamma \kappa) \Phi^T D \Phi + (1-\lambda)\sum_{m=0}^{\infty} \lambda^m \gamma^{m+1} \Phi^T D (I- P^{m+1}) \Phi  \right]({\theta} -{\theta_{\lambda}^*}) \\
    & = ({\theta} -{\theta_{\lambda}^*})^T \left[ \left((1- \gamma \kappa) +(1-\lambda)\sum_{m=0}^{\infty} \lambda^m \gamma^{m+1}\right)  \Phi^T D \Phi - (1-\lambda)\sum_{m=0}^{\infty} \lambda^m \gamma^{m+1} \Phi^T D
    P^{m+1} \Phi  \right]({\theta} -{\theta_{\lambda}^*})\\
    & = ({\theta} -{\theta_{\lambda}^*})^T \left[ \left((1- \gamma \kappa) + \gamma \frac{1-\lambda}{1-\gamma \lambda}\right)  \Phi^T D \Phi - (1-\lambda)\sum_{m=0}^{\infty} \lambda^m \gamma^{m+1} \Phi^T D
    P^{m+1} \Phi  \right]({\theta} -{\theta_{\lambda}^*})\\
    & = ({\theta} -{\theta_{\lambda}^*})^T \left[ \Phi^T D \Phi - (1-\lambda)\sum_{m=0}^{\infty} \lambda^m \gamma^{m+1} \Phi^T D
    P^{m+1} \Phi  \right]({\theta} -{\theta_{\lambda}^*}).
\end{align*}
\end{small}

In the above sequence of equations, the first equality is from the definition of the function $f^{(\lambda)}(\theta)$; the second line comes from plugging in (\ref{eq:vthetadef}) and (\ref{eq:direxpr}); the third equality from breaking the sum in the second term into two pieces, one of which is then absorbed into the first term; the fourth equality follows by using the sum of a geometric series; and the last equality by the definition of $\kappa$ from the theorem statement, which, recall, is $\kappa = (1-\lambda)/(1- \gamma \lambda)$. 

By comparing the expression for $f^{(\lambda)}(\theta)$ we have just derived to (\ref{eq:tdoptdiff}), it is immediate that $-\bar{x}(\theta)$ is a splitting of the gradient of $f^{(\lambda)}(\theta)$.
\end{proof}

\section{Proof of Corollary \ref{thm:bound}}
We will be using standard notation for the mixing time of the Markov Chain: 
\begin{equation}\label{eq:taumix}
    \tau^{{\rm mix}}(\epsilon) = \min \left \{t \in \mathbb{N}, t \geq 1 |  m \rho^t \leq \epsilon \right\}.
\end{equation}We will find it convenient to use several observations made in \cite{bhandari2018finite}. First, Lemma 6 of that paper says that, under  the assumptions of Corollary \ref{thm:bound}, we have that
\begin{equation}
    \|g_t(\theta_t)\|_2 \leq G = r_{\rm max}+2R. \label{eq:gbound}
\end{equation} This holds with probability one; note, however, that because the number of states and actions is finite, this just means one takes the maximum over all states and actions to obtain this upper bound. 

A second lemma from \cite{bhandari2018finite} deals with a measure of ``gradient bias,'' the quantity $\zeta_t (\theta) = (\bar{g}(\theta) - g_t(\theta))^T (\theta^*-\theta)$. As should be unsurprising, what matters in the analysis is not the natural measure of gradient bias, e.g., $\bar{g}(\theta) - g_t(\theta)$, but rather how the angle with the direction to the optimal solution is affected, which is precisely what is measured by $\zeta_t(\theta)$. We have the following upper bound.  
\begin{lemma}[Lemma 11 in \cite{bhandari2018finite}]
Consider a non-increasing step-size sequence, $\alpha_0 \geq \alpha_1 \geq \cdots \geq \alpha_T$. Fix any $t<T$, and set $t^* = \max \{0 , t-\tau^{{\rm mix}}(\alpha_T) \}$. Then 
\begin{equation*}
    E[\zeta_t(\theta_t)] \leq G^2 \left(4 + 6 \tau^{{\rm mix}}(\alpha_T)\right)\alpha_{t^*}.
\end{equation*} \label{lemm:bh}
\end{lemma}

With these preliminaries in place, we are now ready to prove the corollary. The proof follows the steps of \cite{sun2018markov} to analyze Markov gradient descent, using the fact that the gradient splitting has the same inner product with the direction to the optimal solution as the gradient. 

\begin{proof}[Proof of Corollary \ref{thm:bound}]
From the projected TD(0) recursion, for any $t$,
\begin{align}
\| \theta^*-\theta_{t+1}\|_2^2 
& =    \| \theta^*-{\rm Proj}_{\Theta_R}\left(\theta_t+\alpha_t g_t(\theta_t)\right) \|_2^2 \notag \\
& \leq \| \theta^*-\theta_t - \alpha_t g_t(\theta_t) \|_2^2 \notag\\
& =  \| \theta^*-\theta_{t}\|_2^2 - 2 \alpha_t {g}_t(\theta_t)^T(\theta^*-\theta_{t})+ \alpha_t^2 \|g_t(\theta_t)\|_2^2 \notag\\
&=     \| \theta^*-\theta_{t}\|_2^2 - 2 \alpha_t \left[  \bar{g}(\theta_t)^T -  (\bar{g}(\theta) - g_t(\theta))^T  \right]  (\theta^* - \theta)+ \alpha_t^2 \|g_t(\theta_t)\|_2^2 \notag\\
&=     \| \theta^*-\theta_{t}\|_2^2 - 2 \alpha_t \bar{g}(\theta_t)^T(\theta^*-\theta_{t}) + 2 \alpha_t \zeta_t (\theta_t) + \alpha_t^2 \|g_t(\theta_t)\|_2^2 \notag\\
& \leq \| \theta^*-\theta_{t}\|_2^2 - 2 \alpha_t \bar{g}(\theta_t)^T(\theta^*-\theta_{t}) + 2 \alpha_t \zeta_t (\theta_t) + \alpha_t^2 G^2.\notag
\end{align}
In the above sequence of equations, all the equalities are just rearrangements of terms; whereas the first inequality follows that the projection onto a convex set does not increase distance, while the second inequality follows by (\ref{eq:gbound}). 

Next we use  Corollary \ref{cor:identity},  rearrange terms, and sum  from $t=0$ to $t=T-1$:
\begin{align*}
     &\sum_{t=0}^{T-1} 2 \alpha_t E \left [(1-\gamma)\|V_{\theta^*}-V_{\theta_t}\|_{D}^2 + \gamma \|V_{\theta^*}-V_{\theta_t}\|_{{\rm Dir}}^2 \right ] \\
\leq & \sum_{t=0}^{T-1} \left( E\left[\| \theta^*-\theta_{t}\|_2^2\right] - E\left[\| \theta^*-\theta_{t+1}\|_2^2\right] \right) + \sum_{t=0}^{T-1}  2 \alpha_t E\left[\zeta_t (\theta_t)\right] + \sum_{t=0}^{T-1} \alpha_t^2 G^2\\
=    & \left(\| \theta^*-\theta_{0}\|_2^2 - E\left[\| \theta^*-\theta_{T}\|_2^2\right] \right) + \sum_{t=0}^{T-1}  2 \alpha_t E\left[\zeta_t (\theta_t)\right] + \sum_{t=0}^{T-1} \alpha_t^2 G^2\\
\leq & \| \theta^*-\theta_{0}\|_2^2+ \sum_{t=0}^{T-1}  2 \alpha_t E\left[\zeta_t (\theta_t)\right] + \sum_{t=0}^{T-1} \alpha_t^2 G^2.
\end{align*}
Now plugging in the step-sizes  $\alpha_0 = \dots = \alpha_T = 1/\sqrt{T}$, it is immediate that
\begin{small}
\begin{equation*}
     \sum_{t=0}^{T-1} E \left [(1-\gamma)\|V_{\theta^*}-V_{\theta_t}\|_{D}^2 + \gamma \|V_{\theta^*}-V_{\theta_t}\|_{{\rm Dir}}^2 \right ]
\leq  \frac{\sqrt{T}}{2} \left( \| \theta^*-\theta_{0}\|_2^2 +G^2 \right) + \sum_{t=0}^{T-1}  E\left[\zeta_t (\theta_t)\right].
\end{equation*}
\end{small}

Using Lemma \ref{lemm:bh}, have that 
\begin{align}
    \sum_{t=0}^{T-1}  E\left[\zeta_t (\theta_t)\right] & \leq \sum_{t=0}^{T-1} G^2 \left(4 + 6 \tau^{{\rm mix}}(\alpha_T)\right)\alpha_{t^*} 
    \nonumber \\
    & = \sqrt{T} G^2\left(4+6 \tau ^{\rm mix}\left(1/\sqrt{T}\right)\right). \label{eq:zeta}
\end{align} Putting all this together and using the convexity of the function $f(\theta)$, we can bound the error at the average iterate as:
\begin{align*}
      E \left[ (1-\gamma)\|V_{\theta^*}-V_{\bar{\theta}_T}\|_{D}^2 + \gamma \|V_{\theta^*}-V_{\bar{\theta}_T}\|_{{\rm Dir}}^2 \right]
\leq & \frac{1}{T} \sum_{t=0}^{T-1} E \left [(1-\gamma)\|V_{\theta^*}-V_{\theta_t}\|_{D}^2 + \gamma \|V_{\theta^*}-V_{\theta_t}\|_{{\rm Dir}}^2 \right ]\\
\leq & \frac{ \| \theta^*-\theta_{0}\|_2^2 +G^2  }{2\sqrt{T}}  + \frac{G^2\left(4+6 \tau ^{\rm mix}\left(1/\sqrt{T}\right)\right)}{\sqrt{T}}\\
=    & \frac{ \| \theta^*-\theta_{0}\|_2^2 + G^2\left(9 + 12\tau ^{\rm mix}\left(1/\sqrt{T}\right) \right)  }{2\sqrt{T}}.
\end{align*}
\end{proof}

\section{Proof of Corollary \ref{thm:3step}} 
Before starting the proof, we will need a collection of definitions, observations, and preliminary lemmas. We organize these into subheadings below.

\noindent {\bf The Dirichlet Laplacian.} Let $L = (L(i,j))_{n \times n}$ be a symmetric matrix in $\mathbb{R}^{n \times n}$ defined as 
$$L(i,j) =\left\{\begin{matrix}
  - (1/2) \left(\pi_i P(i,j) + \pi_{j}P(j,i)\right) & \text{if } i \neq j\\ 
  \sum_{i' \neq i}|L(i,i')| &  \text{if } i=j
\end{matrix}\right. .$$
It is immediate that the diagonal elements of $L$ are positive and its rows sum to zero. 

Furthermore, it can be shown that for any vector $x$, we have that $\| x \|_{\rm Dir}^2 =  x^T L x $. Indeed: 
\begin{align*}
    x^T L x & = \sum_{i =1}^n \left [\sum_{j \neq i} -\frac{1}{2} \left(\pi_i P(i,j) + \pi_{j}P(j,i)\right) x(i) x(j) +  \left(\sum_{j \neq i} \frac{1}{2} \left(\pi_i P(i,j) + \pi_{j}P(j,i)\right) \right) x(i)^2 \right]\\
    & = \sum_{i < j} \frac{1}{2} \left(\pi_i P(i,j) + \pi_{j}P(j,i)\right) (x(i)-x(j))^2 \\
    & =\frac{1}{2} \sum_{i,j \in [n]}  \frac{1}{2}  \left(\pi_i P(i,j) + \pi_{j}P(j,i)\right) (x(i)-x(j))^2 \\
    & =\frac{1}{2} \sum_{i,j \in [n]} \pi_i P(i,j)  (x(i)-x(j))^2 = \| x \|_{\rm Dir}^2.
\end{align*}

\noindent {\bf Connection to the reversed chain.} We remark that the matrix $L$ is connected to the so-called ``additive reversibilization'' of the matrix $P$, which we explain next. For a stochastic matrix $P$ with stationary distribution $\pi$, it is natural to define the matrix $P^*$ as 
\[ [P^*]_{ij} = \frac{\pi(j)}{\pi(i)} P_{ji}. \] It is possible to verify that the matrix $P^*$ has the same stationary distribution as the matrix $P$ (see \citet{aldous1995reversible}). Intuitively, the equality 
\[ \pi(i) [P^*]_{ij} = \pi(j) P_{ji},\] means that it is natural to interpret $P^*$ as the ``reversed'' chain of $P$: for all pairs $i,j$, the link from $i$ to $j$ is traversed as often under the stationary distribution in $P^*$ as the link from $j$ to $i$ in $P$. 

It can then be shown that the matrix $Q = (P+P^*)/2$ is reversible (see \cite{aldous1995reversible}); this matrix is called the ``additive reversibilization'' of the matrix $P$. It is easy to see that $Q = I - D^{-1} L$; indeed, both the leg-hand side and the right-hand side have the same off-diagonal entries and have rows that sum to one. Because $Q$ is reversible, its spectrum is real. 

The matrix $D^{-1} L$ is clearly similar to the symmetric matrix $D^{-1/2} L D^{-1/2}$ and thus has a real spectrum, with all the eigenvalues nonnegative. Moreover, $D^{-1} L$ has an eigenvalue of zero as $D^{-1} L {\bf 1} =0 $.  As a consequence of these two observations, if we denote by  $r(P)$ the spectral gap of the matrix $Q$, then we have  
\begin{equation} \label{eq:rbound}  r(P) = \frac{1}{1-\lambda_2(Q)} = \frac{1}{\lambda_{n-1}(D^{-1} L)}, 
\end{equation} where $\lambda_{n-1}(D^{-1} L)$ is the second smallest eigenvalue of $D^{-1} L$.

\noindent {\bf Equivalence of norms on ${\bf 1}^\perp$.} We will need to pass between the $||\cdot||_D$ norm and the $||\cdot||_{\rm Dir}$ norm. To that end, we have the following lemma.

\begin{lemma}\label{lem:cn}
For any $x$ with $ \left \langle x,\bm{1} \right \rangle _{D} = 0 $, we have that \begin{equation*}
    \|x\|_{D}^2 \leq r(P) \|x\|_{{\rm Dir}}^2.
\end{equation*}
\end{lemma}

\begin{proof}

Indeed, 
\begin{equation*}
    \underset{\left \langle x,\bm{1} \right \rangle _{D} = 0}{\min}\frac{\|x\|_{{\rm Dir}}^2}{ \|x\|_{D}^2} = \underset{\left \langle x,\bm{1} \right \rangle _{D} = 0}{\min}\frac{x^T L x }{\left \langle x,x \right \rangle _{D}} = \underset{\left \langle x,\bm{1} \right \rangle _{D} = 0}{\min}\frac{\left \langle x,D^{-1}L x \right \rangle _{D} }{\left \langle x,x \right \rangle _{D}} .
\end{equation*}

We next observe that the matrix $D^{-1} L$ is self adjoint in the $\langle \cdot, \cdot \rangle_D$ inner product: 
\[ \langle x, D^{-1} L y \rangle_D = x^T L y = \langle D^{-1} L x , y \rangle .\]  
Since the smallest eigenvalue of $D^{-1} L$ is zero with associated eigenvector of ${\bf 1}$, by the Rayleigh-Ritz theorem we have
\[ \underset{\left \langle x,\bm{1} \right \rangle _{D} = 0}{\min}\frac{\left \langle x,D^{-1}L x \right \rangle _{D} }{\left \langle x,x \right \rangle _{D}}  = \lambda_{n-1} (D^{-1}L). \] 
Putting it all together, we obtain 
$$\frac{\|x\|_{{\rm Dir}}^2}{ \|x\|_{D}^2} \geq \lambda_{n-1}(D^{-1} L) = r(P)^{-1},$$
where the last step used (\ref{eq:rbound}). This completes the proof. 
\end{proof}

\noindent {\bf Error in mean estimation.}  

Recall that we set $\hat{V}_{T}$ be an estimate for the mean of value function in Algorithm 1. Our next lemma upper bounds the error in the estimate $\hat{V}_{T}$. 

\begin{lemma}\label{lem:error_in_mean}
Suppose that $\hat{V}_{T}$ is generated by Algorithm 1 and $\bar{V}= \pi^T V$ denote the mean of value function. Let $t_0 = \max \left \{ t \in \mathbb{N}| t_0 \leq 2 \tau^{{\rm mix}}\left(\frac{1}{2(t_0+1)} \right) \right\}$. Then, for $t > t_0$, we have
\begin{equation*}
      E \left[ \left( \hat{V}_{t}-\bar{V}\right)^2 \right] \leq O\left(\frac{ r_{\rm max}^2 \tau^{{\rm mix}}\left( \frac{1}{2(t+1)} \right) }{(1-\gamma)^2 t}\right).
\end{equation*}{}
\end{lemma}
\begin{proof}
By the definition of $\hat{V}_{t}$ and $\bar{A}_t$ given in Algorithm 1, we can write the recursion: 
\begin{equation*}
    \hat{V}_{t} = \frac{\bar{A}_t}{1-\gamma} = \frac{1}{1-\gamma}\left[ \bar{A}_{t-1} + \frac{1}{t+1} (r_t- \bar{A}_{t-1})\right] =  \hat{V}_{t-1} + \frac{1}{t+1} \left(\frac{r_t}{1-\gamma} - \hat{V}_{t-1}\right).
    \end{equation*}
We next use this recursion to argue:
\begin{align}
    E \left[ \left( \hat{V}_{t}-\bar{V}\right)^2 \right] = &  E \left[ \left(\hat{V}_{t-1} + \frac{1}{t+1} \left(\frac{r_t}{1-\gamma} - \hat{V}_{t-1}\right) -\bar{V} \right)^2 \right] \nonumber \\
     =&   E \left[ \left( \hat{V}_{t-1}-\bar{V}\right)^2  + \frac{1}{(t+1)^2} \left(\frac{r_t}{1-\gamma} - \hat{V}_{t-1}\right)^2 + \frac{2}{t+1} \left(\frac{r_t}{1-\gamma} - \hat{V}_{t-1} \right)\left( \hat{V}_{t-1}-\bar{V}\right) \right] \nonumber \\
     =&   E \left[ \left( \hat{V}_{t-1}-\bar{V}\right)^2  + \frac{1}{(t+1)^2}\left(\frac{r_t}{1-\gamma} - \hat{V}_{t-1}\right)^2+ \frac{2}{t+1} \left(\frac{r_t}{1-\gamma} -\bar{V}- \hat{V}_{t-1} +\bar{V} \right)\left( \hat{V}_{t-1}-\bar{V}\right) \right] \nonumber \\
     =&   E \left[ \left( 1- \frac{2}{t+1}\right)\left( \hat{V}_{t-1}-\bar{V}\right)^2  + \frac{1}{(t+1)^2}\left(\frac{r_t}{1-\gamma} - \hat{V}_{t-1}\right)^2+ \frac{2}{t+1} \left(\frac{r_t}{1-\gamma} -\bar{V} \right)\left( \hat{V}_{t-1}-\bar{V}\right) \right]. \label{eq:basicrecur}
\end{align}

To bound the second term on the right-hand side of (\ref{eq:basicrecur}), we will use that, since $r_{\rm max}$ is the upper bound on absolute values of the rewards, we have that
\begin{equation*}
    \left(\frac{r_t}{1-\gamma} - \hat{V}_{t-1}\right)^2 \leq \left( \frac{r_{\rm max}}{1-\gamma} + \frac{r_{\rm max}}{1-\gamma} \right)^2 = \frac{4r_{\rm max}^2}{(1-\gamma)^2}.
\end{equation*}

We next analyze the third term on the right-hand side of (\ref{eq:basicrecur}). 
Let $\tau_t = \tau^{{\rm mix}}\left(\frac{1}{2(t+1)}\right)$ so that for any state $s''$,  
\begin{equation} \label{eq:taubound}
    \sum_{s=1}^n |P^{\tau_t}(s'',s) -\pi_{s}| = 2 d_{\rm TV} (P^{\tau_t}(s'',\cdot), \pi) \leq 2 m \rho^{\tau_t} \leq \frac{1}{T+1}.
\end{equation}
 We have that
\begin{align*}
    E\left[\left(\frac{r_t}{1-\gamma} -\bar{V} \right)\left( \hat{V}_{t-1}-\bar{V}\right) \right] = & E\left[\left(\frac{r_t}{1-\gamma} -\bar{V} \right)\left( \hat{V}_{t-1}-\hat{V}_{t-1-\tau_t} +\hat{V}_{t-1-\tau_t} - \bar{V} \right) \right] \\
    = & E\left[\left(\frac{r_t}{1-\gamma} -\bar{V} \right)\left( \hat{V}_{t-1}-\hat{V}_{t-1-\tau_t}\right)\right] +E\left[\left(\frac{r_t}{1-\gamma} -\bar{V} \right)\left(\hat{V}_{t-1-\tau_t} - \bar{V} \right) \right]. \\
\end{align*}
We now bound each of the two terms in the last equation separately. For the first term, we have
\begin{align*}
    E\left[\left(\frac{r_t}{1-\gamma} -\bar{V} \right)\left( \hat{V}_{t-1}-\hat{V}_{t-1-\tau_t}\right)\right] \leq & \frac{2r_{\rm max}}{1-\gamma} \sum_{d=t-\tau_t}^{t-1} E \left[ |\hat{V}_{d} -\hat{V}_{d-1}| \right] \\  = &  \frac{2r_{\rm max}}{1-\gamma} \sum_{d=t-\tau_t}^{t-1} \frac{1}{d+1} E \left[  \left|\frac{r_d}{1-\gamma} -\hat{V}_{d-1} \right| \right]\\
    \leq & \frac{4r_{\rm max}^2}{(1-\gamma)^2} \sum_{d=t-\tau_t}^{t-1} \frac{1}{d+1} \\ \leq &  O \left( \frac{\tau_t r_{\rm max}^2}{(1-\gamma)^2(t+1)} \right).\\
\end{align*}
where the last inequality follows from $t > 2 \tau_t$ (which in turn follows from $t \geq t_0$). 

For the second term, we denote the following sigma algebra $\chi^{t}$ denote the sigma algebra generated by the information collected by time $t$, i.e., by the random variables $s_0, r_0, \theta_0, \cdots, s_t, r_t,\theta_t $. We then have that
\begin{align*}
    E\left[\left(\frac{r_t}{1-\gamma} -\bar{V} \right)\left(\hat{V}_{t-1-\tau_t} - \bar{V} \right) \right] = &  E\left[ E \left[\left(\frac{r_t}{1-\gamma} -\bar{V} \right)\left(\hat{V}_{t-1-\tau_t} - \bar{V} \right)| \chi^{t-1-\tau_t} \right] \right] \\
    = & E\left[ \sum_{s=1}^n \left(\frac{\sum_{s'=1}^n P(s,s') r(s, s') }{1-\gamma} -\bar{V} \right)\left(\hat{V}_{t-1-\tau_t} - \bar{V} \right)P^{\tau_t}(s_{t-1-\tau_t},s)  \right] \\
    = & E\left[ \sum_{s=1}^n \left( \frac{\sum_{s'=1}^n P(s,s') r(s, s')}{1-\gamma}- \bar{V}\right) \left(\hat{V}_{t-1-\tau_t} - \bar{V} \right) \left( P^{\tau_t}(s_{t-1-\tau_t},s) -\pi_{s} + \pi_{s} \right)  \right] \\ 
    = & E\left[ \sum_{s=1}^n \left( \frac{\sum_{s'=1}^n P(s,s') r(s, s')}{1-\gamma} - \bar{V} \right) \left(\hat{V}_{t-1-\tau_t} - \bar{V} \right) \left( P^{\tau_t}(s_{t-1-\tau_t},s) -\pi_{s} \right) \right] \\
    & + E\left[ \sum_{s=1}^n \left( \frac{\sum_{s'=1}^n  P(s,s') r(s, s')}{1-\gamma} - \bar{V} \right) \left(\hat{V}_{t-1-\tau_t} - \bar{V} \right) \pi_s \right]  \\
    = &  E\left[ \sum_{s=1}^n \left( \frac{\sum_{s'=1}^n P(s,s') r(s, s')}{1-\gamma} - \bar{V} \right) \left(\hat{V}_{t-1-\tau_t} - \bar{V} \right) \left( P^{\tau_t}(s_{t-1-\tau_t},s) -\pi_{s} \right) \right] \\
    & + E\left[ \left(\hat{V}_{t-1-\tau_t} - \bar{V} \right)  \right] \sum_{s=1}^n \left( \frac{\sum_{s'=1}^n  P(s,s') r(s, s')}{1-\gamma} - \bar{V} \right) \pi_s   \\
     = &  E\left[ \sum_{s=1}^n \left( \frac{\sum_{s'=1}^n P(s,s') r(s, s')}{1-\gamma} - \bar{V} \right) \left(\hat{V}_{t-1-\tau_t} - \bar{V} \right) \left( P^{\tau_t}(s_{t-1-\tau_t},s) -\pi_{s} \right) \right] \\
    & + 0  \\
    \leq & \frac{4 r_{\rm max}^2}{(1-\gamma)^2(t+1)} \\
    \leq & O\left(\frac{ r_{\rm max}^2}{(1-\gamma)^2(t+1)}\right).
\end{align*} Here the first equality follows by iterating conditional expectation; the second, third, fourth, and fifth equality is just rearranging terms; the sixth equality follows from (\ref{eq:vbar}); and the next inequality follows from  (\ref{eq:taubound}) as well as the fact that all rewards are upper bounded by $r_{\rm max}$ in absolute value.

Combining all the inequalities, we can conclude that as long as $t>t_0$, we have that 
\begin{equation*}
     E \left[ \left( \hat{V}_{t}-\bar{V}\right)^2 \right] \leq \left( 1- \frac{2}{t+1}\right) E \left[ \left( \hat{V}_{t-1}-\bar{V}\right)^2 \right] + O \left( \frac{\tau_t r_{\rm max}^2}{(1-\gamma)^2(t+1)^2} \right).
\end{equation*}Let $b_t = O \left( \frac{\tau_t r_{\rm max}^2}{(1-\gamma)^2} \right)$; then the above equation can be compactly written as  
\begin{equation*}
    E \left[ \left( \hat{V}_{t}-\bar{V}\right)^2 \right] \leq \left( 1- \frac{2}{t+1}\right) E \left[ \left( \hat{V}_{t-1}-\bar{V}\right)^2 \right] + \frac{b_t}{(t+1)^2}.
\end{equation*}

Let  $C_t=\max \left \{ (t_0 +1) \left( \hat{V}_{t_0}-\bar{V}\right)^2,b_t \right \}$. We will prove by induction that $t \geq t_0$, $$ E \left[ \left( \hat{V}_{t}-\bar{V}\right)^2 \right] \leq \frac{C_t}{t+1}.$$ Indeed, the assertion holds for $t=t_0$. Suppose that the assertion holds at time $t$, i.e., suppose that  $E \left[ \left( \hat{V}_{t}-\bar{V}\right)^2 \right] \leq C_t/(t+1)$. Then,
\begin{align*}
    E \left[ \left( \hat{V}_{t+1}-\bar{V}\right)^2 \right]  \leq & \left( 1- \frac{2}{t+2}\right) \frac{C_t}{t+1} + \frac{b_t}{(t+2)^2}\\
    = & \frac{C_{t+1}}{t+2} + \left( 1- \frac{2}{t+2}\right) \frac{C_t}{t+1} + \frac{b_t}{(t+2)^2} - \frac{C_{t+1}}{t+2}\\
    = & \frac{C_{t+1}}{t+2} + \frac{C_t (t+2)^2 - 2 C_t (t+2) +b_t(t+1) -C_{t+1} (t+1)(t+2)}{(t+1)(t+2)^2}\\
 = & \frac{C_{t+1}}{t+2} + \frac{\left(C_t -C_{t+1} \right) (t+1)(t+2) +(b_t-  C_t) (t+1) -  C_t}{(t+1)(t+2)^2}\\
 \leq & \frac{C_{t+1}}{t+2},
\end{align*}where the last inequality follows because $C_t \leq C_{t+1}$, $b_t \leq C_t$ and $C_t \geq 0$. Therefore, we have that, for $t \geq t_0$,
\begin{equation*}
    E \left[ \left( \hat{V}_{t}-\bar{V}\right)^2 \right] \leq \frac{C_t}{t+1}.
\end{equation*}{} Since $\left( \hat{V}_{t_0}-\bar{V}\right)^2 \leq 4 \frac{r_{\rm max}^2}{(1-\gamma)^2}$ with probability one, and by definition $t_0 \leq 2 \tau^{\rm max} \left( \frac{1}{2(t_0+1)} \right)$, we have that   $C_t = O \left(\frac{\tau_t r_{\rm max}^2}{(1-\gamma)^2} \right)$ for $t \geq t_0$; this completes the proof.
\end{proof}

With all these preliminary lemmas in place, we can now give the main result of this section, the proof of Corollary \ref{thm:3step}.

\begin{proof}[Proof of Corollary \ref{thm:3step}]
By the Pythagorean theorem, we have
\begin{equation}\label{eq:V_split}
    \|V'_T-V\|_{D}^2 = \|\pi^T V'_{T}\bm{1}-\pi^T V{\bm{1}}\|_{D}^2 + \|V'_{T,\bm{1}^{\perp}}-V_{\bm{1}^{\perp}}\|_{D}^2,
\end{equation}where $V'_{T,\bm{1}^{\perp}}$, $V_{\bm{1}^{\perp}}$ are the projections of $V'_T$, $V$ onto $\bm{1}^{\perp}$ in the $\left \langle \cdot,\cdot \right \rangle _{D}$ inner product. 

Recall, that, in  Algorithm 1, we defined   $$V'_T = V_{\bar{\theta}_T} + \left( \hat{V}_T - \pi^T V_{\bar{\theta}_T} \right) \bm{1}.$$ Therefore, 
\begin{align*}
    \pi^T V'_{T}\bm{1} & =\pi^T V_{\bar{\theta}_T} \bm{1} +\pi^T  \bm{1}\left( \hat{V}_T - \pi^T V_{\bar{\theta}_T} \right) {\bf 1}  \\ &  =\pi^T V_{\bar{\theta}_T} \bm{1} + \hat{V}_T \bm{1} - \pi^T V_{\bar{\theta}_T} \bm{1} \\ 
    & = \hat{V}_T \bm{1}.
\end{align*} Plugging this as well as $\bar{V}= \pi^T V$ into (\ref{eq:V_split}) we obtain:
\begin{equation} \label{eq:ortho}
     \|V'_T-V\|_{D}^2 
=   \|\hat{V}_{T}\bm{1}-\bar{V} \bm{1}\|_{D}^2 + \|V'_{T,\bm{1}^{\perp}}-V_{\bm{1}^{\perp}}\|_{D}^2.
\end{equation}
For the first term on the right hand side of (\ref{eq:ortho}), by the definition of the square norm under $\pi$, it is immediate that
\begin{equation*}
    \| \hat{V}_{T}\bm{1}-\bar{V} \bm{1}\|_{D}^2 = \sum_{i=1}^n \pi_i (\hat{V}_{T}-\bar{V})^2 = \left( \hat{V}_{T}-\bar{V}\right)^2.
\end{equation*}For the second term on the right hand side of (\ref{eq:ortho}), we have
\begin{align*}
    \|V'_{T,\bm{1}^{\perp}}-V_{\bm{1}^{\perp}}\|_{D}^2 = & \| V'_{T, {\bf 1}^{\perp}}  - V_{{\theta}^*,\bm{1}^{\perp}} + V_{{\theta}^*,\bm{1}^{\perp}} - V_{\bm{1}^{\perp}} \|_D^2 \\ 
\leq & 2 \|V'_{T,\bm{1}^{\perp}}-V_{{\theta}^*,\bm{1}^{\perp}}\|_{D}^2+ 2 \| V_{{\theta}^*,\bm{1}^{\perp}} -V_{\bm{1}^{\perp}}\|_{D}^2 \\
\leq & 2 r(P) \|V'_{T,\bm{1}^{\perp}}-V_{{\theta}^*,\bm{1}^{\perp}}\|_{{\rm Dir}}^2 + 2 \| V_{{\theta}^*} -V\|_{D}^2 \\
= & 2 r(P) \|V_{\bar{\theta}_T}-V_{{\theta}^*}\|_{{\rm Dir}}^2 + 2 \| V_{{\theta}^*} -V\|_{D}^2 ,
\end{align*} where the third line follows by the Lemma \ref{lem:cn} and the Pythagorean theorem and the fourth line comes from the observation that $\| \cdot\|_{{\rm Dir}}$ does not change when
we add a multiple of $\bm{1}$.

Combining these results and taking expectation of (\ref{eq:ortho}), we obtain
\begin{align}
     E \left[  \|V'_T-V\|_{D}^2 \right]  
\leq & E \left[ ( \hat{V}_{T}-\bar{V})^2 \right] + 2 r(P) E \left[ \|V_{\bar{\theta}_T}-V_{{\theta}^*}\|_{{\rm Dir}}^2\right] + 2 E \left[\| V_{{\theta}^*} -V\|_{D}^2\right] \nonumber \\
\leq &  O \left( \frac{\tau^{\rm mix} \left( \frac{1}{2(T+1)}\right) r _{\rm max}^2}{(1-\gamma)^2T} \right) +  r(P) \cdot \frac{  \| \theta^*-\theta_{0}\|_2^2+ \left(9+12 \tau ^{\rm mix}\left(1/\sqrt{T}\right)\right) G^2  }{\gamma \sqrt{T}} + 2 \| V_{{\theta}^*} -V\|_{D}^2, \label{eq:fistbound}
\end{align}
where the second inequality follows by Lemma \ref{lem:error_in_mean} and Eq. (\ref{eq:dir}) from the main text.

On the other hand,
\begin{small}
\begin{align}
      E \left[  \|V'_T-V\|_{D}^2 \right] 
= &  E \left[  \|V_{\bar{\theta}_T}-V\|_{D}^2 \right] +  E \left[  \left\| \left( \hat{V}_T - \pi^T V_{\bar{\theta}_T} \right) \bm{1} \right\|_{D}^2 \right] - 2 E \left[ \left(\pi^T V-\pi^T V_{\bar{\theta}_T}\right) \left( \hat{V}_T - \pi^T V_{\bar{\theta}_T} \right)\right] \nonumber  \\
= &  E \left[  \|V_{\bar{\theta}_T}-V\|_{D}^2 \right] +  E \left[  \left\| \left( \hat{V}_T - \pi^T V_{\bar{\theta}_T} \right) \bm{1} \right\|_{D}^2 \right] - 2 E \left[ \left(\bar{V}-\pi^T V_{\bar{\theta}_T}\right) \left( \hat{V}_T - \pi^T V_{\bar{\theta}_T} \right)\right]  \nonumber \\
= & E \left[  \|V_{\bar{\theta}_T}-V\|_{D}^2 \right] +   E \left[   \left( \hat{V}_T - \pi^T V_{\bar{\theta}_T} \right)^2 \right] - 2 E \left[ \left(\hat{V}_T-\pi^T V_{\bar{\theta}_T  }+\bar{V}-\hat{V}_T\right) \left( \hat{V}_T - \pi^T V_{\bar{\theta}_T} \right)\right] \nonumber \\
= & E \left[  \|V_{\bar{\theta}_T}-V\|_{D}^2 \right] -   E \left[   \left( \hat{V}_T - \pi^T V_{\bar{\theta}_T} \right)^2 \right]+ 2 E \left[ \left(\hat{V}_T - \bar{V}\right) \left( \hat{V}_T - \pi^T V_{\bar{\theta}_T} \right)\right] \nonumber \\
\leq & E \left[  \|V_{\bar{\theta}_T}-V\|_{D}^2 \right] -   E \left[   \left( \hat{V}_T - \pi^T V_{\bar{\theta}_T} \right)^2 \right]+ E \left[ \left(\hat{V}_T - \bar{V}\right)^2 + \left( \hat{V}_T - \pi^T V_{\bar{\theta}_T} \right)^2\right] \nonumber \\
= & E \left[  \|V_{\bar{\theta}_T}-V\|_{D}^2 \right] + E \left[ \left(\hat{V}_T - \bar{V}\right)^2\right] \nonumber\\
\leq & 2 E \left[\|V_{\bar{\theta}_T}-V_{\theta^*}\|_{D}^2 \right]+ 2 E \left[\|V_{{\theta}^*}-V\|_{D}^2\right] + E \left[ \left(\hat{V}_T - \bar{V}\right)^2\right] \nonumber \\
\leq & \frac{ 2 \left [ \| \theta^*-\theta_{0}\|_2^2+ \left(9+12 \tau ^{\rm mix}\left(1/\sqrt{T}\right)\right) G^2  \right ]}{(1-\gamma) \sqrt{T}}+ 2 E \left[  \|V_{{\theta}^*}-V\|_{D}^2\right] + O\left(\frac{ r_{\rm max}^2 \tau^{{\rm mix}}\left( \frac{1}{2(T+1)} \right) }{(1-\gamma)^2 T}\right). \label{eq:secondbound} 
\end{align}
\end{small}

Here the first four equalities come from rearranging; the next inequality comes from the identity $2ab \leq a^2+b^2$; the next equality comes from cancellation; the next inequality uses $||u+v||_D^2 \leq 2 ||u||_D^2 + 2 ||v||_D^2$; and the final inequality uses Corollary \ref{thm:bound} and Lemma \ref{lem:error_in_mean}.

We have just derived two bounds on $E[||V_T'-V||_D^2]$, one in Eq. (\ref{eq:fistbound}) and one in Eq. (\ref{eq:secondbound}). We could, of course, take the minimum of these two bounds. We then obtain:
\begin{align*}
E \left[  \|V'_T-V\|_{D}^2 \right] \leq & 2  \|V_{{\theta}^*}-V\|_{D}^2 + O \left( \frac{\tau^{\rm mix} \left( \frac{1}{2(T+1)}\right) r _{\rm max}^2}{(1-\gamma)^2T} \right) \\
&+  \min \left \{  r(P) \cdot \frac{  \| \theta^*-\theta_{0}\|_2^2+ \left(9+12 \tau ^{\rm mix}\left(1/\sqrt{T}\right)\right) G^2 }{\gamma \sqrt{T}}, \right.\\
& \quad \quad \quad \quad \left. \frac{ 2 \| \theta^*-\theta_{0}\|_2^2+ 2\left(9+12 \tau ^{\rm mix}\left(1/\sqrt{T}\right)\right) G^2  }{(1-\gamma) \sqrt{T}}  \right\}. 
\end{align*}Therefore, 
\begin{small}
\begin{equation*}
     E \left[  \|V'_T-V\|_{D}^2 \right]
\leq   2   \|V_{{\theta}^*}-V\|_{D}^2 +  O \left(   \frac{\tau^{\rm mix} \left( \frac{1}{T+1}\right) r _{\rm max}^2}{(1-\gamma)^2T} \right)
 +   \frac{  \| \theta^*-\theta_{0}\|_2^2 + G^2\left[1+\tau ^{\rm mix}(1/\sqrt{T})\right] }{\sqrt{T}}   \cdot \min \left \{ \frac{r(P)}{\gamma} , \frac{2}{1-\gamma}  \right \} ,
\end{equation*}
\end{small}and the proof is  complete. 
\end{proof}
\section{Error Bound for TD with Eligibility Traces} 

We now analyze the performance of projected  TD($\lambda$) which updates as 
\begin{equation}\label{eq:pro_td_lam}
    \theta_{t+1} = {\rm Proj}_{\Theta_{\lambda}}(\theta_t + \alpha_t \delta_t \hat{z}_t),
\end{equation}  where we now use  
$$z_t = \sum_{k=0}^t (\gamma \lambda)^k \phi(s_{t-k}).$$ We remark that this is an abuse of notation, as previously $z_t$ was defined with the sum starting at negative infinity, rather than zero; however, in this section, we will assume that the sum starts at zero. The consequence of this modification of notation is that Theorem \ref{thm:bound_lambda} does not imply that $-E[z_t]$ is the gradient splitting of an appropriately defined function anymore, as now one needs to account for the error term coming from the beginning of the sum. 

We assume $\Theta_{\lambda}$ is a convex set containing the optimal solution $\theta_{\lambda}^*$. We will further assume that the norm of every element in $\Theta_{\lambda}$ is at most $R_{\lambda}$. Recall that

We begin by introducing some notation. Much of our analysis follows \cite{bhandari2018finite} with some deviations where we appeal to Theorem \ref{thm:bound_lambda}, and the notation bellow is mostly identical to what is used in that paper. First, recall that  we denote the quantity $\delta_t z_t$ by $x(\theta_t,z_t)$. We define $\zeta_t(\theta,z_t)$ as a random variable which can be thought of as a measure of the bias that TD($\lambda$) has in estimation of the gradient: $$\zeta_t(\theta,z_t) = (\bar{x}(\theta) - \delta_tz_t)^T(\theta_{\lambda}^*-\theta).$$ Analogously to the TD(0) case, what turns out to matter for our analysis is not so much the bias per se, but the inner product of the bias with the direction of the optimal solution as in the definition of $\zeta_t(\theta, z_t)$. 

We will next need an upper bound on how big $||x(\theta, z_t)||_2$ can get. Since under Assumption  \ref{ass:features}, we have that $||\phi(s)||_2 \leq 1$ for all $s$, we have that 
\[ \|z_t\|_2 \leq \frac{1}{1-\gamma \lambda}. \]   Furthermore, we have that  $$|\delta_t| = \left|r(s,s')+\gamma \phi(s')^T \theta_t -\phi(s)^T \theta_t \right| \leq r_{{\rm max}} +2R_{\lambda},$$ where we used $|r(s,s')| \leq r_{{\rm max}}$ as well as Cauchy-Schwarz. Putting the last two equations together, we obtain  
\begin{equation}\label{eq:gbound_lambda}
    \|x(\theta,z_t)\|_2 \leq\frac{r_{{\rm max}} +2R_{\lambda}}{1-\gamma \lambda} :=  G_{\lambda}.
\end{equation}
Compared to the result for TD($0$), the bound depends on a slightly different definition of the mixing time that takes into account the geometric weighting in the eligibility trace. Define $$\tau_{\lambda}^{{\rm mix}}(\epsilon) = \max \{ \tau^{{\rm MC}}(\epsilon), \tau^{{\rm Algo}}(\epsilon) \},$$where  
\begin{align*}\tau^{{\rm MC}}(\epsilon) & = \min \left \{t \in \mathbb{N}_0 |  m \rho^t \leq \epsilon \right\} \\ \tau^{{\rm Algo}}(\epsilon) & = \min \left \{t \in \mathbb{N}_0 | (\gamma \lambda) ^t \leq \epsilon \right\}
\end{align*}

The main result of this section is the following corollary of Theorem \ref{thm:lambda_iden}.

\begin{corollary}\label{thm:bound_lambda}
   Suppose Assumptions \ref{ass:mc}-\ref{ass:features} hold. Suppose further that $(\theta_t)_{t \geq 0}$ is generated by the  Projected TD($\lambda$) algorithm of (\ref{eq:pro_td_lam}) with  $\theta_{\lambda}^*$ belonging to the convex set $\Theta_{\lambda}$ and step-sizes $\alpha_0 = \dots = \alpha_T = 1/\sqrt{T}$. Then
 $$E \left[ f^{(\lambda)}(\theta) \right] 
     \leq  \frac{ \| \theta_{\lambda}^*-\theta_{0}\|_2^2 +G_{\lambda}^2\left[14 + 28\tau_{\lambda} ^{{\rm mix}}\left(1/\sqrt{T}\right) \right]}{2\sqrt{T}},$$ where the function $f^{(\lambda)}(\theta)$ was  defined in Theorem \ref{thm:bound}.
\end{corollary}

\begin{proof}
 We begin with the standard recursion for the distance to the limit:
\begin{small}
\begin{align*}
\| \theta_{\lambda}^*-\theta_{t+1}\|_2^2 
& =    \| \theta_{\lambda}^*-{\rm Proj}_{{\Theta}_{\lambda}}(\theta_t+\alpha_t \delta_t z_t) \|_2^2  \\
& \leq \| \theta_{\lambda}^*-\theta_t - \alpha_t \delta_t z_t \|_2^2 \\
&=     \| \theta_{\lambda}^*-\theta_{t}\|_2^2 - 2 \alpha_t \delta_t z_t^T(\theta_{\lambda}^*-\theta_{t}) + \alpha_t^2 \|\delta_t z_t\|^2_2 \\
&=     \| \theta_{\lambda}^*-\theta_{t}\|_2^2 - 2 \alpha_t \left(\bar{x}(\theta_t)^T -\left(\bar{x}(\theta_t)^T-\delta_t z_t^T\right)\right)(\theta_{\lambda}^*-\theta_{t}) + \alpha_t^2 \|\delta_t z_t\|^2_2 \\
& = \| \theta_{\lambda}^*-\theta_{t}\|_2^2 - 2 \alpha_t \left(\bar{x}(\theta_t) -\bar{x}(\theta_{\lambda}^*)  \right)^T(\theta_{\lambda}^*-\theta_{t}) + 2 \alpha_t \zeta_t (\theta_t,z_t) + \alpha_t^2  \|x(\theta_t,z_t)\|^2_2 \\
& = \| \theta_{\lambda}^*-\theta_{t}\|_2^2 - 2 \alpha_t \left[  (1-\gamma \kappa)\|V_{\theta}-V_{\theta_{\lambda}^*}\|_{D}^2 + (1-\lambda) \sum_{m=0}^{\infty} \lambda^m \gamma^{m+1} \|V_{\theta}-V_{\theta_{\lambda}^*}\|_{{\rm Dir},m+1}^2 \right] \\
 & \quad + 2 \alpha_t \zeta_t (\theta_t,z_t) + \alpha_t^2 \|x(\theta_t,z_t)\|^2_2 \\
& \leq \| \theta_{\lambda}^*-\theta_{t}\|_2^2 - 2 \alpha_t \left[  (1-\gamma \kappa)\|V_{\theta}-V_{\theta_{\lambda}^*}\|_{D}^2 + (1-\lambda) \sum_{m=0}^{\infty} \lambda^m \gamma^{m+1} \|V_{\theta}-V_{\theta_{\lambda}^*}\|_{{\rm Dir},m+1}^2 \right] \\
 & \quad + 2 \alpha_t \zeta_t (\theta_t,z_t) + \alpha_t^2  G_{\lambda}^2.
\end{align*}
\end{small}In the  sequence of equations above the first inequality follows that the projection onto a convex set does not increase distance; the remaining equalities are rearrangements, using the quantity $\bar{x}(\theta)$ defined in (\ref{eq:xbar}),  that $\bar{x}(\theta_{\lambda}^*) =0$ from (\ref{eq:xbar^*}), and Proposition \ref{prop:splitmotivation}; and the final inequality used (\ref{eq:gbound_lambda}).  

We next take expectations, rearrange terms, and sum: 
\begin{align*}
     &\sum_{t=0}^{T-1} 2 \alpha_t E \left [(1-\gamma \kappa)\|V_{\theta}-V_{\theta_{\lambda}^*}\|_{D}^2 + (1-\lambda) \sum_{m=0}^{\infty} \lambda^m \gamma^{m+1} \|V_{\theta}-V_{\theta_{\lambda}^*}\|_{{\rm Dir},m+1}^2 \right ] \\
\leq & \sum_{t=0}^{T-1} \left( E[\| \theta_{\lambda}^*-\theta_{t}\|_2^2] - E[\| \theta_{\lambda}^*-\theta_{t+1}\|_2^2] \right) + \sum_{t=0}^{T-1}  2 \alpha_t E[\zeta_t (\theta_t,z_t)] + \sum_{t=0}^{T-1} \alpha_t^2 G_{\lambda}^2\\
=    & \left( \| \theta_{\lambda}^*-\theta_{0}\|_2^2 - E[\| \theta_{\lambda}^*-\theta_{T}\|_2^2] \right) + \sum_{t=0}^{T-1}  2 \alpha_t E[\zeta_t (\theta_t,z_t)] + \sum_{t=0}^{T-1} \alpha_t^2 G_{\lambda}^2\\
\leq & \| \theta_{\lambda}^*-\theta_{0}\|_2^2+ \sum_{t=0}^{T-1}  2 \alpha_t E[\zeta_t (\theta_t,z_t)] + \sum_{t=0}^{T-1} \alpha_t^2 G_{\lambda}^2.
\end{align*}
Plugging in the step-sizes $\alpha_0 = \dots = \alpha_T = 1/\sqrt{T}$, we obtain
\begin{align*}
     &\sum_{t=0}^{T-1} E \left [(1-\gamma \kappa)\|V_{\theta}-V_{\theta_{\lambda}^*}\|_{D}^2 + (1-\lambda) \sum_{m=0}^{\infty} \lambda^m \gamma^{m+1} \|V_{\theta}-V_{\theta_{\lambda}^*}\|_{{\rm Dir},m+1}^2 \right ] \\
\leq & \frac{\sqrt{T}}{2} \left( \| \theta_{\lambda}^*-\theta_{0}\|_2^2 +G_{\lambda}^2 \right) + \sum_{t=0}^{T-1}  E[\zeta_t (\theta_t,z_t)].
\end{align*}
Using Lemma 20 in \citep{bhandari2018finite}, we have that 
\begin{align*}
    \sum_{t=0}^{T-1}  E[\zeta_t (\theta_t,z_t)] & \leq 6 \sqrt{T} \left (1+ 2 \tau_{\lambda} ^{\rm mix}(\alpha_T)\right)G_{\lambda}^2 + \sum_{t=0}^{2\tau_{\lambda} ^{\rm mix}(\alpha_T)}(\gamma \lambda)^t G_{\lambda}^2 \\
    & \leq 6\sqrt{T}  \left(1+ 2 \tau_{\lambda} ^{\rm mix}(\alpha_T)\right)G_{\lambda}^2 + \left(2\tau_{\lambda} ^{\rm mix}(\alpha_T) + 1\right) G_{\lambda}^2.
\end{align*} Combining with convexity, we get
\begin{align*}
     & E \left[(1-\gamma \kappa)\|V_{\theta_{\lambda}^*}-V_{\bar{\theta}_T}\|_{D}^2 + (1-\lambda) \sum_{m=0}^{\infty} \lambda^m \gamma^{m+1} \|V_{\theta_{\lambda}^*}-V_{\bar{\theta}_T}\|_{{\rm Dir},m+1}^2 \right]\\
\leq & E \left [ \frac{1}{T} \sum_{t=0}^{T-1} (1-\gamma \kappa)\|V_{\theta}-V_{\theta_{\lambda}^*}\|_{D}^2 + (1-\lambda) \sum_{m=0}^{\infty} \lambda^m \gamma^{m+1} \|V_{\theta}-V_{\theta_{\lambda}^*}\|_{{\rm Dir},m+1}^2 \right ]\\
\leq & \frac{  \| \theta_{\lambda}^*-\theta_{0}\|_2^2 +G_{\lambda}^2   }{2\sqrt{T}}  + \frac{6 \sqrt{T} \left(1+ 2 \tau_{\lambda} ^{\rm mix}(\alpha_T)\right)G_{\lambda}^2 + \left(2\tau_{\lambda} ^{\rm mix}(\alpha_T) + 1\right) G_{\lambda}^2}{T}\\
\leq    & \frac{ \| \theta_{\lambda}^*-\theta_{0}\|_2^2 + G_{\lambda}^2\left(14 + 28\tau ^{\rm mix}\left(1/\sqrt{T}\right) \right)  }{2\sqrt{T}}.
\end{align*}
\end{proof}

\end{document}